\documentclass{article}

\usepackage[preprint]{neurips_2026}

\usepackage[utf8]{inputenc} 
\usepackage[T1]{fontenc}    
\usepackage{hyperref}       
\usepackage{url}            
\usepackage{booktabs}       
\usepackage{amsfonts}       
\usepackage{nicefrac}       
\usepackage{microtype}      
\usepackage{xcolor}         

\title{Efficient Exploration Is Enough}

\usepackage{natbib}
\usepackage{amsmath}
\usepackage{amsthm}
\usepackage{mathrsfs}
\usepackage{graphicx}
\usepackage{wrapfig}
\usepackage{dsfont} 
\usepackage{datetime}
\usepackage{hyperref}
\usepackage{derivative}
\usepackage{caption}
\usepackage{subcaption}
\usepackage{algorithm}
\usepackage{algorithmic}
\usepackage{multirow}

\usepackage{amssymb}
\usepackage{mathtools}
\theoremstyle{plain}
\newtheorem{theorem}{Theorem}[section]
\newtheorem{proposition}[theorem]{Proposition}

\newtheorem{definition}[theorem]{Definition}
\newtheorem{assumption}[theorem]{Assumption}
\newtheorem{remark}[theorem]{Remark}

\author{%
  Mikel Malagón${^1}$\thanks{Correspondence to: Mikel Malagón \texttt{<mikel.malagon@ehu.eus>}.},
  Jon Vadillo$^1$,
  Josu Ceberio$^1$,
  Michael Bowling$^{2,3,4}$,
  Jose A. Lozano$^{5}$ \\
  $^1$University of the Basque Country (UPV/EHU) \\
  $^2$University of Alberta \\
  $^3$Alberta Machine Intelligence Institute \\
  $^4$Canada CIFAR AI Chair \\
  $^5$Basque Center for Applied Mathematics (BCAM) \\
}

\newcommand{\algocomment}[1]{\textcolor{gray}{\textit{\# #1}}}

\newcommand{\minibox}[1]{%
  \begingroup
  \setlength{\fboxsep}{0pt}
  \setlength{\fboxrule}{0.2pt}
  \fbox{\textcolor[HTML]{#1}{\rule{1.5ex}{1.5ex}}}%
  \endgroup
}

\newcommand{\actions}{\mathcal{A}}
\newcommand{\observations}{\mathcal{O}}
\newcommand{\states}{\mathcal{S}}
\newcommand{\env}{e}
\newcommand{\envset}{\mathcal{E}}
\newcommand{\policy}{\lambda}

\newcommand{\polunif}{\lambda_{\mathrm{unif}}}
\newcommand{\poldet}{\lambda_{\mathrm{det}}}
\newcommand{\policyset}{\Lambda}
\newcommand{\policysetall}{\Lambda_\text{all}}
\newcommand{\detpolicyset}{\Lambda_\text{det}}
\newcommand{\tranf}{{p_e}}

\newcommand{\hist}{\mathcal{H}}
\newcommand{\rvhist}{H} 
\newcommand{\hpre}[1]{h_{1:#1}}

\newcommand{\paramoptimiset}{\Omega}
\newcommand{\paramoptimi}{\omega}
\newcommand{\setwm}{\mathcal{F}}
\newcommand{\wmfunc}{f}

\newcommand{\ece}{\varepsilon}

\newcommand{\discrep}{d}
\newcommand{\loss}{\ell}
\newcommand{\lossg}{\loss^\text{global}_t}

\newcommand{\freq}{c}  
\newcommand{\binomial}{\text{Binomial}}

\newcommand{\variance}[2]{\text{Var}_{#1}\left[ #2 \right]}
\newcommand{\expect}[2]{\mathbb{E}_{#1}\left[ {#2} \right]}
\newcommand{\bias}[1]{\text{Bias}\left( {#1} \right)}

\newcommand{\envweightS}{{\nu^{t-1}_\tranf}}
\newcommand{\devi}{\phi}
\newcommand{\population}{\mathcal{P}}

\newcommand{\gini}{\mathscr{V}_\tranf}
\newcommand{\unifbias}{\mathscr{B}_\tranf}

\begin{document}

\maketitle


\begin{abstract}
 This work introduces an alternative view of efficient exploration and studies its theoretical and empirical implications in the absence of extrinsic rewards. Specifically, we define efficient explorers as agents that prioritize generating generalizable experience, i.e., data that supports learning models capable of predicting and adapting across the environment. This allows us to analyze efficient exploration through the lens of prediction and generalization. Theoretically, we demonstrate that optimally efficient explorers naturally schedule their trajectories to visit the most informative and learnable regions first. Empirically, we show that optimizing for these agents gives rise to an automatic curriculum of progressively more complex behaviors, even in relatively simple environments. These results indicate that pursuing this purely intrinsic objective alone is enough to drive the emergence of highly sophisticated behaviors. We believe that this new framework provides a principled mechanism by which agent-environment systems may sustain an open-ended process of increasingly complex behavior without external rewards, tasks, or objectives.



\end{abstract}



\section{Introduction} \label{sec:introduction}

Agents---understood as systems acting by themselves according to certain goals or norms in an environment \citep{barandiaran2009defining}---constitute the substrates of intelligent life as we know it.
The class of systems that satisfy this definition is broad, yet exploration is a ubiquitous behavior across all of them.
As such, exploration has been a topic of great focus in the field of Reinforcement Learning (RL) \citep{sutton2018rlbook}.
The field has produced a rich taxonomy of exploration strategies, ranging from classical methods such as $\epsilon$-greedy, Thompson sampling \citep{thompson1933likelihood}, and upper confidence bounds \citep{auer2002ucb}, to intrinsic motivation approaches widely used in deep RL, including count-based exploration \citep{bellemare2016unifying}, Intrinsic Curiosity Module (ICM) \citep{pathak2017curiosity}, and Random Network Distillation (RND) \citep{burda2018exploration}.

Most of these works are shaped by the idea that exploration implies maximizing environment coverage, i.e., visiting as many distinct regions of it as possible \citep{strehl2008analysis,henaff2022exploration}. Furthermore, methods that do not pursue full coverage typically restrict their focus to regions or skills relevant to a specific downstream task \citep{farebrother2026jumpy} or balance exploration with extrinsic reward. Thus, in the absence of extrinsic objectives, efficient exploration is broadly understood as efficient uniform environment coverage \citep{thrun1992efficient}. However, we believe this perspective overlooks the role of generalization. Experience gathered in one region might allow accurately predicting dynamics in other unseen scenarios. When models can generalize, or with proper priors, uniform environment coverage---even strategies that maximize the immediate prediction error of a model \citep{schmidhuber1991possibility,shyam2019model}---no longer represent the most efficient approach to exploration. Building on this idea, we propose a new definition of efficient exploration:
\begin{center}
    \textit{``An agent efficiently explores its environment if it \textbf{prioritizes} generating the most \\ \textbf{generalizable experience}.''}
  \end{center}

Specifically, by \textit{experience} we refer to the agent's history of interactions (i.e., observations and actions), while \textit{generalizable} implies a model and a broader target distribution to generalize toward. This model---commonly known as a \textit{world model} \citep{ha2018worldmodels}---is learned to estimate the environment's dynamics from the agent's ongoing stream of experience, and should ideally predict any scenario that might occur within the environment as well as possible. Thus, this purely intrinsic efficient exploration objective requires agents to generate histories that reduce the global prediction error of the world model over the long run.

To illustrate, consider an agent embodied as a humanoid robot on an island. Under our definition, an agent that learns to walk robustly across diverse terrains would be an efficient explorer, as effective locomotion is a prerequisite for generating diverse, generalizable experience. Conversely, an agent that engages in reckless behavior risks damaging itself, prematurely ending its ability to gather data. Thus, survival naturally emerges as a necessary condition for efficient exploration: it maximizes the time available to interact with the environment, ensuring the world model receives sufficient data to generalize as broadly as possible. Ultimately, progressively more complex behaviors---such as swimming to nearby islands, building shelters to ensure long-term survival, or learning to predict the weather---would naturally arise as mechanisms to further increase exploration efficiency.
This contrasts with the \textit{reward is enough} hypothesis by \citet{silver2021reward}, which argues that intelligent behaviors, including exploration, are ultimately driven by the maximization of extrinsic rewards. Here, we explore an alternative perspective, and provide theoretical and empirical evidence that efficient exploration---a purely intrinsic signal---is enough to drive the emergence of increasingly complex and sophisticated behaviors, including survival mechanisms traditionally thought to be explicitly motivated by extrinsic reward.

The remainder of this paper offers a particular formalization of efficient exploration and investigates its theoretical and empirical implications in the absence of extrinsic rewards, where an agent's optimality is solely determined by its exploration efficiency. Theoretically, this framework allows us to analyze efficient exploration through the lens of prediction and model generalization. First, we demonstrate that efficient exploration can  be achieved by deterministic, targeted behaviors. Second, under standard Markovian assumptions and a tabular world model, we show that exploration efficiency requires agents to schedule their visits such that the most informative and learnable regions of the environment are visited first, delaying visits to hard-to-model regions and naturally avoiding the noisy-TV problem or absorbing states. Furthermore, we establish connections between efficient exploration, survival, and prior work on intrinsic motivation.
Empirically, we investigate the behaviors that drive efficient exploration in partially observable settings where agents and world models are parametrized by Neural Networks (NNs) and formal theoretical analysis is intractable.
To this end, we employ a simple Monte Carlo (MC) approximation method to optimize agents towards efficient exploration.
When observing these efficient explorers, we find that complex behaviors automatically emerge, including exploiting environmental symmetries, solving simple mazes, and employing sophisticated navigation strategies. Furthermore, we note that this optimization process naturally induces sequences of progressively more complex behaviors.
Consequently, we believe this work represents a step toward the understanding and principled generation of complex behaviors in the absence of extrinsic goals for natural and artificial agents. Due to space constraints, a thorough review of the literature and detailed connections to prior work are deferred to Appendix~\ref{apx:relworks}.

\section{Preliminaries} \label{sec:prelims}

\paragraph{Notation} Uppercase Greek letters and calligraphic uppercase letters denote sets ($\Lambda$, $\mathcal{A}$), and lowercase Latin and Greek letters denote variables and functions ($\alpha$, $f$), except for the indicator function $\mathds{1}$. Uppercase Latin letters are reserved for random variables ($X$). Conditional probability functions are denoted as $p: \mathcal{X} \to \Delta(\mathcal{Y})$, where $\Delta(\mathcal{Y})$ is the probability simplex over $\mathcal{Y}$. Thus, $p(\cdot \mid x)$ is a conditional distribution over $\mathcal{Y}$ conditioned on $x\in\mathcal{X}$. We write $p(y \mid x)$ for the probability value assigned to $y$ given $x$.

%
%
We denote by $\actions$ and $\observations$ the sets of all possible actions and observations respectively, which we assume are finite.
%
%
Following the development of \citet{abel2023definition}, we refer to \textbf{histories} as sequences of action and observation pairs $(a, o)$ generated by the interactions between an agent and its environment, where the set of all possible histories of any length is $\hist = \bigcup_{t=0}^\infty (\actions \times \observations)^t$. We use $\hist_t = (\actions\times\observations)^t$ to refer to the set of all histories $\hpre{t}=(a_1 o_1,\ldots,a_t o_t)$ of $t$ interactions, where $\hpre{0}=\emptyset$ is the empty history. Given $\hpre{t}$ and $i \leq t$, the history $\hpre{i}$ denotes the prefix of length $i$ of the history $\hpre{t}$, thus $\hpre{i} = (a_1 o_1,\ldots,a_{i} o_i)$.
Following \citet{abel2025plasticity}, we employ the convention of \textbf{actions preceding observations}, that is, action $a_i$ causes observation $o_i$.
\begin{definition} \label{def:policy}
An \textbf{agent} is a function $\policy: \hist \rightarrow \Delta(\actions)$ that maps histories to probability distributions over actions.
\end{definition}
\begin{definition}\label{def:env}
An \textbf{environment} is a function $\env: \hist \times \actions \rightarrow \Delta(\observations)$ that maps tuples of histories and actions to probability distributions over observations.
\end{definition}

Let  $\policysetall$ denote the set of all possible agents,  $\detpolicyset \subseteq \policysetall$ be the subset of deterministic agents $\detpolicyset = \left\{ \policy \in \policysetall\, :\, \forall_{h\in\hist}\exists_{a\in\actions}\,  \policy(a\mid h) = 1 \right\}$, and $\envset$ be the space of all environments. We denote by $\policyset$ a non-empty subset $\policyset \subseteq \policysetall$ of agents of interest. Where specified, we assume contains $\policyset$ all possible agents for histories of length $t$.

\begin{assumption}\label{asm:pol-set-coverage}
  $\forall_{\hpre{t}\in\hist_t} \forall_{\policy\in\policysetall} \exists_{\policy'\in\policyset}$ such that $\forall_{i\in[1,t]}\; \policy(\hpre{i-1}) = \policy'(\hpre{i-1})$.
\end{assumption}

Furthermore, note that the definitions above are general and capture both the Markovian and non-Markovian cases, fully and partially observable settings, and episodic and non-episodic settings as discussed by \citet{abel2023definition}.
At this point, we can define the probability distribution over histories of length $t$ induced by following an agent $\policy\in\policysetall$ in an environment $\env\in\envset$ as,
\begin{equation}
    p(\hpre{t}\mid\policy, \env) = \prod_{i=1}^{t} \env(o_i\mid\hpre{i-1}, a_i)\, \policy(a_i\mid\hpre{i-1}),
\label{eq:prob-h-lambda-e}
\end{equation}
where $\hpre{i-1}$ is the prefix of length $i-1$ of $\hpre{t}$ and $a_i$ is the $i$-th action in $\hpre{t}$. We denote by $\rvhist_{1:t}$ the random variable corresponding to histories following this distribution.

\begin{figure}
    \centering
    \includegraphics[width=0.9\linewidth]{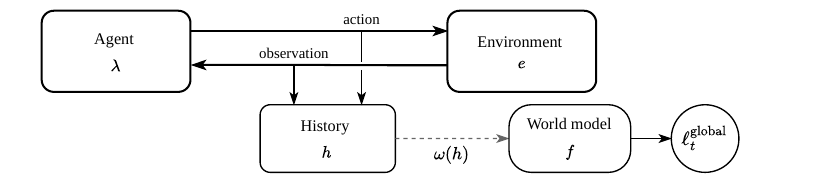}
    \caption{Diagram of the considered setting.}
    \label{fig:main}
    \vskip -1em
\end{figure}

As mentioned in Section~\ref{sec:introduction}, we evaluate an agent's exploration efficiency by the generalizability of the experience it generates. As the agent interacts with the environment, its ongoing stream of experience (i.e., its history) is used to dynamically estimate a world model at each interaction step. Importantly, we do not require that the agent explicitly possesses or accesses this world model. Rather, the model serves primarily as an external construct to quantify the agent's exploration efficiency based on the model's ability to generalize over time (see Figure~\ref{fig:main}). Nevertheless, our formalism is sufficiently general to accommodate agents that do act based on an internal world model learned from experience. To formalize these concepts, we define world models analogously to environments as follows.

\begin{definition}
A \textbf{world model} is a function $\wmfunc: \hist \times \actions \rightarrow \Delta(\observations)$ that maps a tuple of a history and an action to a probability distribution over observations.
\end{definition}

Let $\setwm$ be the set of all world models (hypothesis space). Note that for an environment $\env$, the optimal world model $f^*$ is given by $f^* = \env$.
Given a history generated by an agent, we formalize the estimation of these models through a generic learning rule.

\begin{definition}\label{def:lear-rule}
A \textbf{learning rule} is a function $\paramoptimi: \hist \rightarrow \setwm$ that maps histories to world models.
\end{definition}

We denote by $\paramoptimiset$ the set of all possible learning rules.
For instance, the learning rule $\tilde{\paramoptimi}\in\paramoptimiset$ that returns a world model that best fits the environment for a particular input history $\hpre{t}$ (not necessarily $f^*$) is given by,
\begin{equation}
    \tilde{\paramoptimi}(\hpre{t}) \in \underset{\wmfunc \in \setwm}{\arg \min} \sum_{i=1}^t \discrep\big(\wmfunc(\cdot \mid \hpre{i-1} , a_i), \ \ e(\cdot \mid \hpre{i-1}, a_i)\big),
\end{equation}
where the discrepancy function $\discrep$ is a divergence or metric that computes the statistical distance between two probability distributions (in this case, over observations). Namely, \mbox{$\discrep: \Delta(\observations) \times \Delta(\observations) \rightarrow \mathbb{R}^+$}, where $\discrep(\wmfunc(\cdot\mid h, a), \env(\cdot\mid h, a)) = 0$ iff $\wmfunc(\cdot\mid h, a) = \env(\cdot\mid h, a)$. For example, $\discrep$ can be instantiated as the total variation distance or the Kullback-Leibler divergence.

In this context, given an environment $\env$ and a world model $\wmfunc$, we define a global loss function that quantifies the overall discrepancy between the probability distributions generated by both functions across all actions and realizable histories up to the time horizon $t$.

\begin{definition}\label{def:gloss}
  Under an environment $e$ and a time horizon $t$, the \textbf{global loss} of a world model $\wmfunc$ is,
\begin{equation}
  \label{eq:global-loss}
  \lossg (\wmfunc) = \sum_{k=0}^{t-1} \expect{H_{1:k}\mid \polunif, \env}{\sum_{a \in \actions} \discrep\big(\wmfunc(\cdot \mid \rvhist_{1:k}, a), \, e(\cdot \mid \rvhist_{1:k}, a)\big)},
\end{equation}
where $\polunif$ is the uniform agent, i.e., $\forall_{h\in\hist,a\in\actions}\, \polunif(a \mid h) = 1 / |\actions|$.
\end{definition}

First, note that Equation~\eqref{eq:global-loss} computes the world model error (i.e., the discrepancy) exactly once for every possible  transition up to step $t$.
Second, for each history, it weights the error of this $(k+1)$-th transition by the probability of encountering the preceding history $\hpre{k}$ under the distribution induced by the uniform agent acting in the environment $\env$. Although evaluating the loss under the uniform policy $\polunif$ might initially seem arbitrary, this agent effectively serves as an unbiased history sampler for the environment $\env$.
\begin{proposition}\label{prop:rhounif}
  Let $\polunif\in\policysetall$ be the agent that always takes actions with uniform probability, $\forall_{h\in\hist,a\in\actions}\, \polunif(a\mid h) = 1 / |\actions| $. Then, $p(\hpre{k} \mid \polunif, \env) = |\actions|^{-k} \prod_{i=1}^{k} \env(o_i\mid\hpre{i-1}, a_i)\,$.
\end{proposition}

The proof of Proposition~\ref{prop:rhounif} is deferred to Appendix~\ref{proof:rhounif}.
Intuitively, this proposition establishes that under $\polunif$, the probability of a history is strictly proportional to the environment's intrinsic transition dynamics. Because the action selection probabilities factor out as the constant $|\actions|^{-k}$, the resulting distribution over histories is governed solely by the environment's likelihood.
This weighting scheme ensures that discrepancies are not treated equally; errors occurring in highly improbable scenarios under $\env$ are naturally discounted, whereas errors in likely scenarios are prioritized.

Finally, note that at each timestep $t$, the world model is estimated exclusively using the interactions the agent has collected so far---that is, the history up to step t, which we refer to as its \textit{experience}. As noted by \citet{javed2024bigworld}, the agent will typically only ever observe a minuscule fraction of the full history space $\hist$ with which to minimize the global loss. Therefore, an efficient explorer must strategically gather its experience to yield the most generalizable data possible. We explore this dynamic in detail in the subsequent sections.



\section{Exploration efficiency} \label{sec:hypothesis}

We define the exploration efficiency of an agent in terms of the \emph{Expected Cumulative Error} (ECE) induced by its behavior, measured via the global predictive error of world models estimated from its generated histories.

\begin{definition} \label{def:ece} Under an environment $\env$, a learning rule $\paramoptimi$, and a time horizon $t\in\mathbb{N}$, the \textbf{expected cumulative error} of an agent $\policy\in\policysetall$ is defined as,
\begin{equation}
    \ece_t(\policy) = \mathbb{E}_{\rvhist_{1:t} \mid \policy, \env}\left[\sum_{i=1}^t \lossg(\paramoptimi(\rvhist_{1:i})) \right],
\label{eq:efficiency}
\end{equation}
where
$\rvhist_{1:i}$ refers to the prefix of $\rvhist_{1:t}$ truncated at timestep $i$.
\end{definition}

Informally, $\ece_t$ aggregates the expected global losses of the sequence of world models generated over $t$ interaction steps. Specifically, the model evaluated at each step $i\in[1, t]$ is estimated using the history collected by executing agent $\policy$ in environment $\env$ up to that step, i.e., $\rvhist_{1:i}$. To simplify the notation, we omit the environment $\env$ and the learning rule $\paramoptimi$ from $\ece_t$, and they are implicitly fixed by the context. Under our framework, agents that achieve a lower ECE correspond to more efficient explorers. Note that we employ $\lossg$ rather than $\ell_i^\text{global}$, meaning that each world model is evaluated across all input histories of lengths $[0, t-1]$.

\begin{definition}
  Let $\policy,\policy'\in\policysetall$, and a time horizon $t\in\mathbb{N}$. The agent $\policy$ is a more \textbf{efficient explorer} than $\policy'$ iff $\ece_t(\policy) < \ece_t(\policy')$.
\end{definition}

Consequently, exploration efficiency is an ordinal property, i.e., an agent is an efficient explorer only relative to others.
%
%
%
%
Thus, within $\policyset \subseteq \policysetall$ and a fixed time horizon $t$, there exists a subset $\policyset^*_t \subseteq \policyset$ of optimal explorers that minimize Equation~\eqref{eq:efficiency}.

\begin{definition} \label{def:optimal-explorers}
For a given time horizon $t$, a learning rule $\paramoptimi$, and an environment of interest $e$, we define the set of \textbf{optimal explorers} as those agents $\policyset^*_t \subseteq \policyset$ that achieve minimum ECE. Namely,
\begin{equation}
    \policyset^*_t = \underset{\policy\in\policyset}{\arg \min}\ \ece_t(\policy).
\label{eq:optimal-explorers}
\end{equation}
\end{definition}

Note that multiple agents may achieve the same minimum ECE value. We therefore define optimal explorers as a set rather than a single agent, as different behaviors can induce histories that are equally informative about the environment up to a time horizon $t$.
For example, starting from the same initial knowledge, two individuals who pursue different fields of study---such as biology and physics---may resolve similar uncertainty about their environment (the real world) over their lifetimes. Although their experiences differ substantially at the behavioral level, both trajectories can produce similarly generalizable models of the environment even under the same learning rule.
Furthermore, in a symmetric environment such as an empty grid world, two agents that explore the space using different patterns (e.g., clockwise versus counterclockwise sweeps) may generate experience with identical generalization properties and achieving the same ECE.
This intuition is captured by the next theoretical properties of the set of optimal explorers. 

\begin{theorem} \label{theo:opt-expl-props}
  Let $\policyset^*_t$  be the set of optimal agents for a time horizon $t$. Given Assumption~\ref{asm:pol-set-coverage}, and for all $t\in\mathbb{N}$ and $\paramoptimi\in\paramoptimiset$, the following hold:
  \begin{enumerate}
  \item There always exists at least one fully deterministic optimal explorer: $\forall_{\env\in\envset}\,, \policyset^*_t \cap \detpolicyset \neq \emptyset$.
  \item There exist environments with stochastic optimal explorers: $\exists_{\env\in\envset}$ $\policyset^*_t \setminus \detpolicyset \neq \emptyset$.
  \end{enumerate}
\end{theorem}

The proof for Theorem~\ref{theo:opt-expl-props} is provided in Appendix~\ref{proof:opt-expl-props}. Intuitively, the first property establishes that efficient exploration does not inherently require randomness; unlike classical heuristics that inject stochastic noise to stumble upon novel states, optimally efficient exploration can be realized as a completely deliberate, systematic, and targeted process of data collection.
At the same time, the second property clarifies that determinism is not a strict requirement. By proving that there exist environments with non-deterministic optimal explorers, we show that when an environment presents multiple equally informative pathways, an agent can act stochastically across them without suffering any loss in its overall exploration efficiency.

\section{Theoretical analysis} \label{sec:theory}

In this section, we theoretically characterize the optimally efficient exploration behavior under standard Markovian dynamics, stationarity, a tabular world model, and the squared $L_2$ norm as discrepancy function.\footnote{In this case, $\discrep(\wmfunc_h(\cdot\mid s, a), \tranf(\cdot\mid s, a)) = \, \parallel \wmfunc_h(\cdot\mid s, a) - \tranf(\cdot\mid s, a) \parallel^2_2$.} Crucially, this tractable setting builds intuition that extends beyond these idealized constraints, as we demonstrate in Section~\ref{sec:experiments}.

We start by assuming a fully observable Markovian environment defined by a controlled Markov process \citep{puterman1994mdp}.

\begin{definition}
\label{def:cMP}
A \textbf{controlled Markov process} (cMP) is a Markov Decision Process (MDP) without the reward function and discount factor. It is defined by the triplet $(\states, \actions, \tranf)$, where $\states$ is the space of states, $\actions$ is the space of actions, and $\tranf$ is the transition probability function $\tranf: \states \times \actions \rightarrow \Delta(\states)$.
\end{definition}

As the cMP is defined over states, we restrict the history definition from Section~\ref{sec:prelims} to use states rather than observations, that is, $\observations = \states$, $\env(\hpre{t}, a) = \tranf(s_t, a)$, and $\hpre{0} = s_0 \in \states$.
%
We also make the following assumption about agents.

\begin{assumption}\label{asm:policy1}
Agents are stationary and Markovian. 
\end{assumption}

Under this assumption, the more general agent definition from Definition~\ref{def:policy} simplifies to  $\policy: \states \rightarrow \Delta(\actions)$. For the remainder of this section, we therefore treat $\policy\in\policysetall$ as a function over states rather than histories.


As introduced in Section~\ref{sec:prelims}, world models estimate the environment ($\tranf$ in this case) from histories generated by an agent. In this section, we analyze a specific instantiation of the learning rule $\paramoptimi$, defining it as a smoothed empirical estimator that derives $\tranf$ directly from the transition frequencies within the input history $h$. Under the Markovian assumption (Definition~\ref{def:cMP}), the resulting world models take the form $f: \states \times \actions \times \hist \rightarrow \Delta(\states)$. For convenience, we let $\wmfunc_h$ denote the model estimated from history $h$, meaning $\paramoptimi(h) = f_h$ throughout this section. Namely,
\begin{equation}\label{eq:wm-freqs}
    \wmfunc_h(s'|s, a) = \frac{\freq'(s', s, a; h) + \alpha}{\freq(s, a; h) + \alpha |\states|},
\end{equation}
where the function $\freq(s, a; h)$ counts the number of times the state-action pair $(s, a)$ occurs in $h$, $\freq'(s', s, a; h)$ denotes the number of times the transition $(s, a, s')$ appears in the history $h$, and $\alpha > 0$.
Specifically, $\freq(s, a; h_t) = \sum_{i=1}^t \mathds{1}(s_i=s \land a_i=a)$ and $\freq'(s', s, a; h_t) = \sum_{i=1}^{t-1} \mathds{1}(s_{i+1}=s' \land s_i=s \land a_i=a)$. Algebraically, Equation~\eqref{eq:wm-freqs} represents the posterior predictive mean of a Dirichlet-Multinomial model with a symmetric prior $\alpha$.
%
%
%
%
Additionally, the Markovian assumption allows us to simplify the global loss $\lossg$ in Equation~\eqref{eq:global-loss} as follows.

\begin{proposition} \label{prop:markov-gloss}
  Under a cMP environment (see Definition~\ref{def:cMP}) whose dynamics are defined by $\tranf$ and agents in $\policysetall$ follow Assumption~\ref{asm:policy1}, Equation~\eqref{eq:global-loss} can be simplified as,
\begin{equation} \label{eq:markov-gloss}
  \lossg (\wmfunc) = \sum_{a,s}  \discrep\big(\wmfunc(\cdot \mid s, a), \, \tranf(\cdot \mid s, a)\big)\,  \envweightS(s)
\end{equation}
where $\envweightS(s) = \sum_{k=0}^{t-1} \Pr(S_k = s \mid \polunif, \tranf)$ and is greater as the state $s$ has greater visiting probability according to the transition function $\tranf$ in the cMP (details in Remark~\ref{rem:envweight}).
\end{proposition}

The proof of Proposition~\ref{prop:markov-gloss} is provided in Appendix~\ref{apx:markov-gloss}.
Finally, by instantiating the discrepancy function $d$ as the squared $L_2$ norm in the global loss (see Equation~\eqref{eq:markov-gloss}), we derive an alternative expression for optimal explorers under the setting introduced in this section.

\begin{theorem}\label{theo:main}
Let the triplet $(\states, \actions, \tranf)$ be a cMP following Definition~\ref{def:cMP} and $\policyset$ be a set of agents under Assumption~\ref{asm:policy1}. Assume further that the discrepancy function $\discrep$  is the squared $L_2$ norm.
Then, the problem of finding the optimal explorers $\policyset_t^*$  is equivalent to solving the following minimization problem,
\begin{align}\label{eq:main-theo-main}
   \underset{\policy\in\policyset}{\arg \min} \sum_{i=1}^t \sum_{a, s} \envweightS(s) & \left( \gini(s, a) \, \expect{N\mid \policy, \tranf}{\frac{N}{(N+\alpha |\states|)^2}} \right. \nonumber \\
   & + \left. \unifbias(s,a)\, \expect{N\mid \policy, \tranf}{\frac{\alpha^2}{(N+\alpha |\states|)^2}}\right),
\end{align}
where $N = \freq(s, a; H_{1:i})$, $\gini(s,a) = \sum_{s'}\tranf(s'\mid s, a) (1-\tranf(s'\mid s, a))$ is the variance of the environment's next state distribution at $(s, a)$, and $\unifbias(s,a) = \sum_{s'} (1-|\states| \tranf(s'\mid s, a))^2$ is the smoothing bias that measures how far the prior of $f$ (see Definition~\ref{eq:wm-freqs}) is from the true next state distribution.
\end{theorem}

We defer the proof of Theorem~\ref{theo:main} to Appendix~\ref{apx:proof-theo-lp}. Under the assumptions of this section, this theorem reformulates the efficient exploration problem (in Equation~\eqref{eq:optimal-explorers}) into an equivalent optimization problem.
Remarkably, this analytical reframing allows us to clearly expose the underlying mechanics of optimally efficient exploration as follows.

\paragraph{Interpretation}
Equation~\eqref{eq:main-theo-main} shows that the contribution of each state-action pair $(s,a)$ to the total ECE is governed by two terms, both weighted by $\envweightS(s)$, the probability of visiting $s$ via a random walk in $t-1$ steps.
The first term inside the brackets is proportional to the variance $\gini(s, a)$, which grows with the entropy of the transition dynamics $\tranf(\cdot \mid s, a)$, and  its expectation decays as $O(1/N)$, where $N$ is the visit count to $(s,a)$.
The second term depends on the bias $\unifbias(s,a)$, which captures the discrepancy between the world model's prior and $\tranf(\cdot \mid s,a)$.
This term behaves inversely: it decreases as transition entropy rises, and its expectation decays more rapidly at $O(N^{-2})$.
Figure~\ref{fig:ece-dif-p} illustrates this dynamic by plotting the ECE contribution of a given $(s, a)$ tuple across different entropy levels of $\tranf(\cdot \mid s, a)$.
\begin{wrapfigure}[14]{r}{0.35\textwidth}
    \centering
    \includegraphics[width=\linewidth]{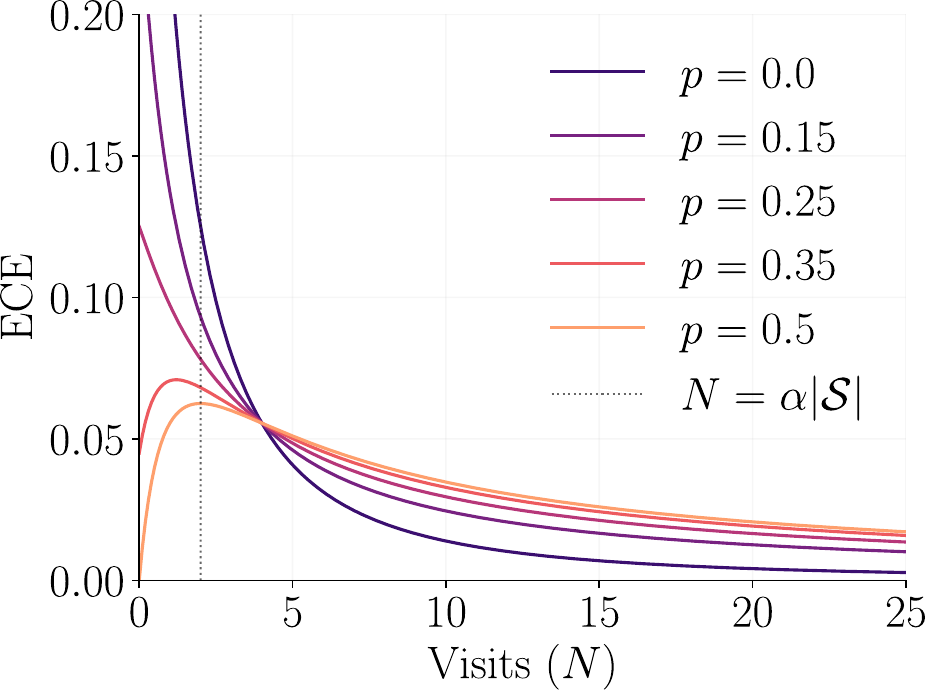}
    \vskip -0.2em
    \caption{ECE in a particular $(s,a)$ tuple, for $\alpha = 1$, $\states = \{s, s'\}$, $\envweightS(s) = \envweightS(s') = 1$, and different values of $p = \tranf(s'\mid s, a)$.}
    \vskip -0.2em
    \label{fig:ece-dif-p}
\end{wrapfigure}
When the entropy of $\tranf(\cdot \mid s, a)$ is very high, the model's prior aligns with the true distribution. As a result, the ECE starts at zero, implicitly discouraging the agent from allocating visits to these $(s,a)$ tuples.
For transitions with high (but non-maximal) entropy, the slower $O(N^{-1})$ variance term dominates, meaning a large number of visits is required to noticeably reduce the error. Conversely, for low-entropy transitions, the rapidly decaying $O(N^{-2})$ bias term governs the error.
Consequently, an efficient explorer seeking to minimize ECE most rapidly will naturally prioritize these easier-to-model, deterministic transitions, strategically allocating visits away from highly stochastic regions where loss reduction is prohibitively slow. Furthermore, in the extreme case of maximal entropy, the prior naturally aligns with reality, resulting in an initial ECE of zero that inherently discourages visitation.


\paragraph{Modeling accuracy vs. environment coverage} The expected loss reduction in Equation~\eqref{eq:main-theo-main} decays at rates of $O(N^{-1})$ and $O(N^{-2})$  with respect to the state-action visitation count $N$, sharing important similarities with count-based intrinsic motivation which often employ $1/\sqrt{N}$ bonuses \citep{sutton2018rlbook,bellemare2016unifying}: the marginal utility of revisiting transitions diminishes over time. Despite this similarity, Equation~\eqref{eq:main-theo-main} implies that an agent might not be able to uniformly distribute visits across the environment without sacrificing its overall ECE. Because inherent learning \emph{costs} vary dramatically across state-action pairs, the agent must strategically under-sample or entirely ignore difficult-to-learn transitions. Instead, it prioritizes its finite interaction budget on regions that lead to rapid and reliable improvements in global predictive accuracy, i.e., $\lossg$.


\paragraph{Survival} Additionally, Theorem~\ref{theo:main} establishes a formal link between efficient exploration and survival. Consider that the environment contains an absorbing state---such as one representing the \textit{death} of the agent. If an agent enters this state at timestep $i$ it is forced to allocate all of its remaining visitation budget $t-i$ to a single state-action pair. Because the expected error in Equation~\eqref{eq:main-theo-main} decays as $O(N^{-2})$ for deterministic transitions (as those of the absorbing state) the marginal reduction in ECE from these redundant visits to the same state quickly vanishes.
In this case, the agent wastes its budget on a learned transition while leaving other learnable regions of the environment unexplored. Therefore, to globally minimize ECE, an optimal explorer must inherently avoid absorbing states, giving rise to emergent survival behaviors as a requirement of efficient exploration.


\section{Empirical analysis} \label{sec:experiments}

Having analyzed the mechanics of efficient exploration under idealized conditions in Section~\ref{sec:theory}, we now investigate its implications within more complex, partially observable environments where agents and world models are parameterized by deep NNs.
Empirical evidence supports that efficient exploration serves as a natural catalyst for complex, purposeful behavior.
Specifically, we show that exploration efficiency---in the absence of extrinsic rewards---drives agents to
spontaneously generate progressively sophisticated behaviors.


\paragraph{Environments} We employ \texttt{SmallWorld}, a custom suite of grid-based environments illustrated in Figures~\ref{fig:results-empty-blocks} and~\ref{fig:results-maze-randcolors}. Rather than RGB pixels, \texttt{SmallWorld} grids are defined by scalar cell values in $[-1,1]$, where negative values denote solid obstacles, zero represents empty space, and positive values indicate entities such as agents. Agents execute discrete movement actions (up, down, left, right) and receive partial, agent-centered observations as square matrices ($n\times n$ local slices of the grid). We use four environments: \textit{Empty}, a simple bare room; \textit{Blocks}, an open area with three solid obstacles; \textit{Maze}, a map partitioned into four rooms; and \textit{RandColors}, which contains two rooms whose cell values change randomly at each timestep, serving as sources of stochastic noise. We use $n = 5$ for all environments except \textit{Empty}, where $n = 7$. See Appendix~\ref{apx:smallworld} for details.

\paragraph{Method} We parameterize the agents using Gated Recurrent Units (GRUs) \citep{cho2014gru} and the world models as feed-forward NNs that process sequences of the past eight interactions, optimized via AdamW \citep{loshchilov2018adamw} (i.e., the learning rule). We use basic MC estimation to approximate the ECE and employ the OpenES \citep{salimans2017evolution} Evolution Strategy (ES) to optimize a population of agents towards more efficient exploration by minimizing this objective. Importantly, this method serves primarily as a simple tool for empirical analysis rather than a primary contribution of this work. See Appendix~\ref{apx:exp-impl} for additional details and hyperparameters.

\subsection{Results} \label{sec:results1}

\begin{figure}
  \centering
  \includegraphics[width=0.90\linewidth]{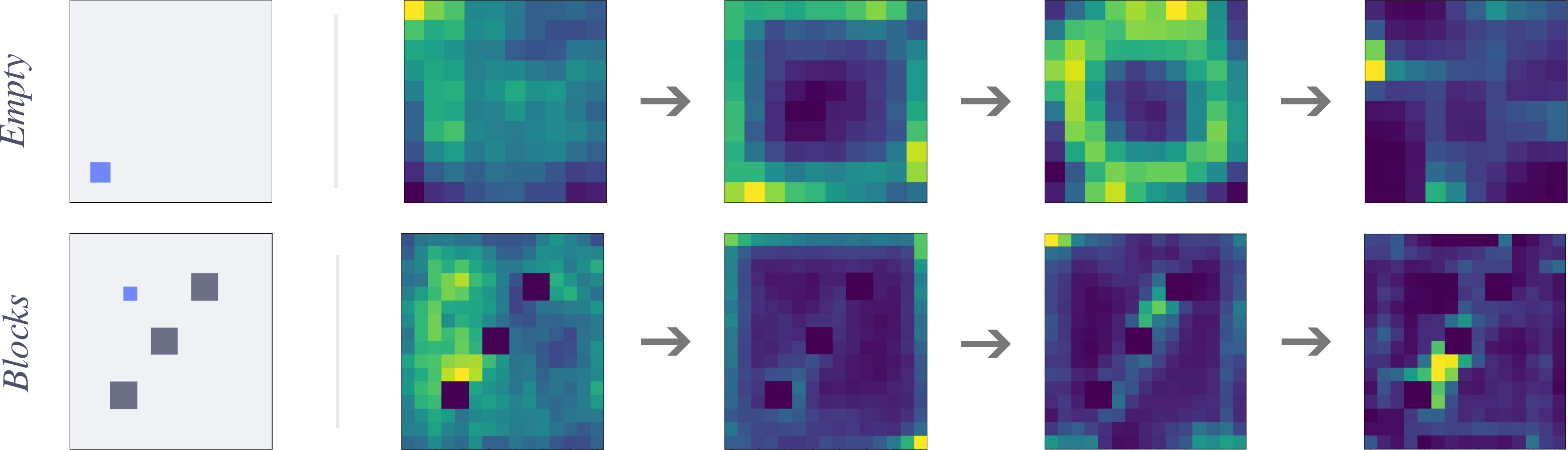}
  \caption{Environments are illustrated on the left, where the agent is denoted in blue \minibox{7287fd}, obstacles in dark gray \minibox{6c6f85}, and traversable space in light gray \minibox{eff1f5}. Corresponding results on the right display sequences of per-cell visitation heatmaps (lighter areas indicate more visits) progressing from early ES generations (left) to more efficiently exploring agents (right).}
  \label{fig:results-empty-blocks}
\vspace{-5mm}
\end{figure}

For the \textit{Empty} and \textit{Blocks} environments, Figure~\ref{fig:results-empty-blocks} illustrates the evolution of per-cell visitation heatmaps, progressing from early generations on the left to the best efficient explorers found by the ES on the right. Focusing on the \textit{Empty} environment, a clear behavioral shift is visible. Initially, agents focus on uniform coverage of the space (leftmost heatmap). However, as the ES finds agents that explore more efficiently, this unstructured exploration shifts into progressively specific trajectories. By the final generation, the agents converge on a helix-like navigation pattern (rightmost heatmap). We attribute this structure to the predictive challenges of partial observability (see Figure~\ref{fig:empty-compare} in Appendix~\ref{apx:more-results}). Specifically, straight traversals with homogeneous floor observations make predicting incoming walls more challenging for the world model. The emergent helix pattern addresses this by actively and periodically sampling these regions of the environment, providing the world model with highly informative and non-redundant data.

A similar shift toward structured behavior occurs in the \textit{Blocks} environment. Here, the most efficient agent found by the ES (rightmost heatmap) adopts a highly targeted strategy that exploits the inherent symmetries of the environment. Because the upper and lower blocks present identical transition dynamics, gathering data from both is redundant. By intensely sampling just one of these regions, the agent acquires the data necessary to accurately generalize to interactions across all structurally similar areas (see Figure~\ref{fig:errors-blocks} in Appendix~\ref{apx:more-results}). This illustrates how efficient exploration naturally penalizes redundancy while driving agents to exploit the regions of the environment that lead to broader generalization, such as structural symmetries.

In the \textit{Maze} environment, the lack of structural symmetries leads to full coverage (Figure~\ref{fig:res-maze} in Appendix~\ref{apx:more-results}). Yet, the most efficient explorers achieve this through a highly structured, cyclic trajectory rather than a random walk. As shown by the agent's initial 45 steps (see plot (b) in Figure~\ref{fig:results-maze-randcolors}), it prioritizes sampling the most informative regions (walls and corners) before transitioning to more predictable open areas. Furthermore, rather than moving to the adjacent bottom-right room, the agent goes to the upper region, avoiding redundant resampling of the starting area. This behavior aligns with our theoretical results: efficient exploration inherently schedules trajectories such that highly informative regions are visited first.

The \textit{RandColors} environment introduces regions of varying modeling difficulty, combining a deterministic corridor with two stochastic rooms, where the irreducible error in the right room is five times the left's. Tracking the best agent's trajectory across timestep intervals, in plot (d) of Figure~\ref{fig:results-maze-randcolors}, we observe that during the first 140 steps, the agent traverses the corridor in a zig-zag trajectory---learning the patterns that generalize across the corridor---and then transitions into the left (low entropy) room. Only later, as seen in plot (e) of Figure~\ref{fig:results-maze-randcolors} (showing steps 350 to 450), the agent shifts to the (higher entropy) right room exclusively. This progression shows that the agent actively constructs a curriculum \citep{bengio2009curriculum} for the world model: it prioritizes sampling the most informative and learnable regions before tackling harder, higher-entropy areas.
As shown in Figure~\ref{fig:res-randcolors} of Appendix~\ref{apx:more-results}, the agent ultimately dedicates more of its overall interaction time to the right room, as reducing the prediction error in this higher-entropy room demands more data (aligning with Figure~\ref{fig:ece-dif-p}).
Consistent with our theoretical results, this confirms that efficient explorers allocate interactions based on modeling difficulty, informativeness, and generalization capability.

\begin{figure}
  \centering
  \includegraphics[width=0.9\linewidth]{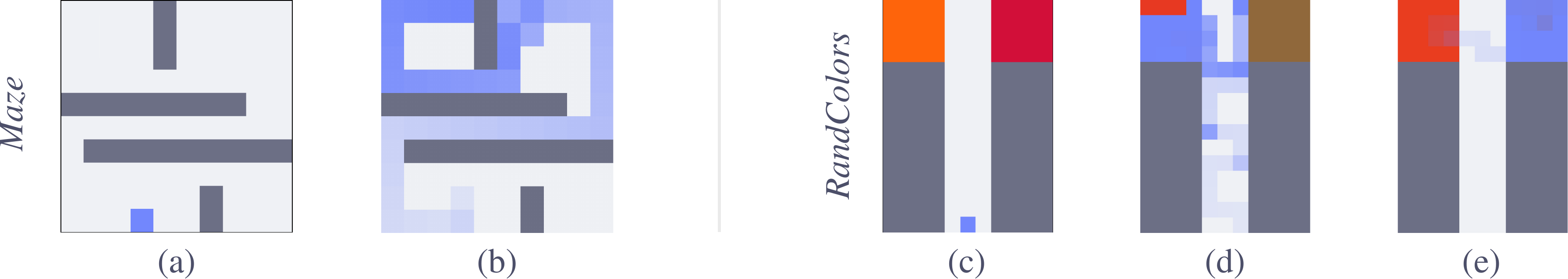}
  \caption{Plots (a) and (c) depict the \textit{Maze} and \textit{RandColors} environments following the same coloring scheme as in Figure~\ref{fig:results-empty-blocks} and including the colors \minibox{fe640b}, \minibox{d20f39}, and \minibox{90683b} as traversable space (i.e., floor).
The rest of the plots show the most efficiently exploring agents' trajectory across specific timesteps in blue (where earlier steps have higher transparency). Plot (b) shows the first 45 timesteps in \textit{Maze}, while (d) and (e) show the timestep intervals 1-140 and 350-450 in \textit{RandColors} respectively.
  }
  \label{fig:results-maze-randcolors}
\vspace{-5mm}
\end{figure}

\vspace{-3mm}
\section{Discussion} \label{sec:discussion}

From foundational RL algorithms \citep{strehl2008analysis,sutton2018rlbook} to current state-of-the-art methods \citep{schulman2017proximal,gallici2025simplifying}, exploration is primarily driven by random noise. In contrast, Theorem~\ref{theo:opt-expl-props} demonstrates that optimal efficient exploration is achieved through deliberate and targeted (even deterministic) data collection, motivating the development of strategies that rely on structure and generalization rather than stochasticity.

Furthermore, Theorem~\ref{theo:main} broadens the principled understanding of the behaviors that lead this optimal efficient exploration under common RL assumptions.
Intuitively, it establishes that an efficient explorer sequentially prioritizes the most learnable and  informative regions of the environment before addressing parts with higher inherent entropy or irreducible error.
Section~\ref{sec:theory} also provides a theoretical connection to survival: behaviors that optimize exploration efficiency, rather than immediate, localized surprise \citep{schmidhuber1991possibility,pathak2017curiosity,burda2018exploration}, naturally lead to survival; a vital property for non-episodic and non-ergodic environments, such as the real world.

Section~\ref{sec:experiments} goes beyond the assumptions of Section~\ref{sec:theory} to the more complex setting of partially observable environments, and NN-based (nonstationary) agents and world models capable of generalization.
Results align with our previous theoretical findings, illustrating how efficient exploration gives rise to sophisticated emergent behaviors that automatically create curricula for world models. For instance, agents actively exploit symmetries in the environment and systematically sequence interactions according to modeling complexity. Additional discussion on these emerging curricula is provided in Appendix~\ref{apx:order-matters}.

Furthermore, the evolutionary process employed to optimize towards efficient exploration in Section~\ref{sec:experiments}  reveals a meta-level curriculum. By tracking the best agents across generations (see Figure~\ref{fig:results-empty-blocks}) of the evolutionary process, we observe behaviors that grow progressively more sophisticated over time. The emergence of highly structured, non-trivial behaviors---such as the helix pattern in the \textit{Empty} environment---suggests that in richer, unconstrained domains, optimizing for efficient exploration could provide a principled mechanism to sustain the open-ended generation of increasingly complex behavior \citep{hughes2024position}. Importantly, unlike recent trends in this area \citep{bauer2023human,faldor2025omniepic}, we show that efficient exploration is enough to drive this open-ended process entirely in the absence of extrinsic objectives or tasks \citep{silver2021reward}. See Appendix~\ref{apx:open-endedness} for an extended discussion on the open-endedness definition by \citet{hughes2024position} and our efficient exploration framework.

Finally, extending our theoretical analysis to accommodate more expressive model classes presents an exciting path forward. In particular, analyzing efficient exploration under world models with generalization properties---such as Gaussian Processes or Neural Tangent Kernels \citep{jacot2018neural}---would offer deeper and rigorous insights into how optimal explorers exploit features of complex environments. Furthermore, while the basic MC estimation of the ECE employed in Section~\ref{sec:experiments} served to observe the behaviors of efficient explorers beyond the assumptions of Section~\ref{sec:theory}, developing scalable approximations for more complex settings remains a promising direction for future work.



\subsubsection*{Acknowledgments}

We are grateful to Xabier Barandiaran, Miguel Aguilera, Iñigo Urteaga, and Borja Calvo for the insightful conversations and feedback in the early stages of the project. We also thank Ainhize Barrainkua and Mahshid Rahmani Hanzaki for reading preliminary versions of this paper, and Jose A. Pascual for the technical support.

Mikel Malagón acknowledges a predoctoral grant from the Spanish MICIU/AEI with code PREP2022-000309, associated with the research project PID2022-137442NB-I00 funded by the Spanish MICIU/AEI/10.13039/501100011033 and FEDER, EU. 
Josu Ceberio has been partially supported by the Spanish MICIU/AEI/10.13039/PID2023-149195NB-I00.

This work is also funded through the BCAM Severo Ochoa accreditation
CEX2021-001142-S/MICIN/AEI/10.13039/501100011033; and the Research Groups 2026-2029 (IT2109-26), the BERC 2026-2029 program and Elkartek (KK-2026/00055 and KK-2026/00087) from the Basque Government. 
Finally, this research was partially funded by the Natural Sciences and Engineering Research Council of Canada (NSERC) and the Canada CIFAR AI Chairs program.

\bibliography{references}


\newpage

\appendix

\medskip

\renewcommand\qedsymbol{$\blacksquare$}

\section{Connections to previous work} \label{apx:relworks}

There exist many connections between the ideas presented in this paper and prior work across different fields. A comprehensive treatment of these connections or formal unification with existing theories is beyond the scope of this paper. Instead, we briefly discuss a few influential perspectives that are most closely related at a conceptual level, drawing informal connections and clarifying key conceptual differences.

\paragraph{Exploration and intrinsic motivation in RL}
Techniques designed to enhance exploration are nearly ubiquitous in reinforcement learning.
Common approaches include injecting noise into the agent’s decision-making process \citep{mnih2015human,fortunato2018noisy,hessel2018rainbow}, entropy regularization \citep{mnih2016asynchronous,schulman2017proximal}, and explicit reward bonuses for exploration \citep{chentanez2004intrinsically,bellemare2016unifying,henaff2022exploration}, commonly referred to as intrinsic motivation.
Methods in the latter category rely on auxiliary signals that quantify some notion of exploration and encourage agents to maximize it.

A prominent class of intrinsic motivation methods are those based on visitation counts, which try to maximize environment coverage.
Early examples include the work of \citet{strehl2008analysis}, which augments the reward with a term proportional to $N(s,a)^{-1/2}$, where $N(s,a)$ denotes the number of times action $a$ has been taken in state $s$. More recently, \citet{bellemare2016unifying} extended this idea to non-tabular settings by approximating state visitation counts using density models. While effective at encouraging coverage \citep{henaff2022exploration} as a form of exploration, these methods do not distinguish between states in terms of their informativeness---under efficient exploration, not all regions of the environment are equally valuable, as some interactions might lead to experience that generalizes more broadly to the underlying structure of the environment or is easier to learn. For instance, see results in Figures~\ref{fig:results-empty-blocks} and \ref{fig:results-maze-randcolors} from Section~\ref{sec:experiments}).

Another widely used exploration signal is model surprise. Early work by \citet{schmidhuber1991possibility}, followed by \citet{chentanez2004intrinsically} and \citet{oudeyer2007intrinsic}, proposed using the prediction error of a world model as an intrinsic reward. Subsequent approaches further developed this idea by measuring disagreement among ensembles of predictive models \citep{shyam2019model}, the error of an inverse dynamics model \citep{pathak2017curiosity}, or the error in predicting the features of a randomly initialized and fixed network \citep{burda2018exploration}. Although these methods employ predictive models to drive exploration, their objectives differ fundamentally from efficient exploration as defined in this work. In particular, maximizing immediate surprise favors local prediction errors, whereas our formulation prioritizes sampling trajectories that improve model generalization across the environment as a whole, even if this entails revisiting (or not) familiar states or reducing prioritizing long-term prediction error over short-term one (e.g., see results for the \textit{Maze} environment in Figure~\ref{fig:results-maze-randcolors}).

\paragraph{Continual reinforcement learning}
\citet{javed2024bigworld} argue that continual RL \citep{abel2023definition} is required for developing agents that are embodied in complex \textit{big} worlds as ours.
Indeed, continual RL has attracted increasing attention in recent years \citep{khetarpal2022towards}. However, many existing approaches \citep{rusu2016progressive, wolczyk2021continual, gaya2023building, malagon2024self} retain central assumptions inherited from the traditional RL paradigm that may be ill-suited to continual RL scenarios \citep{abel2024three,elelimy2025rethinking}.

In particular, \citet{elelimy2025rethinking} identify four foundational assumptions of conventional RL that conflict with the goals of continual RL \citep{abel2023definition}: the Markov decision process formalism, reliance on atemporal artifacts such as fixed training and deployment phases, evaluation via expected cumulative reward, and episodic benchmarks. In contrast, the framework we develop in Sections~\ref{sec:prelims} and~\ref{sec:hypothesis} deliberately avoids these assumptions by proposing formalisms that do not rely on episodic environments, explicit reward functions (of any type), Markov assumptions, or predefined task boundaries. Moreover, we argue that efficient exploration is a sufficient (intrinsic) objective for continual learning agents to develop increasingly complex behaviors autonomously with minimal restrictions.

\paragraph{Reward-free reinforcement learning}
Reward-free reinforcement learning commonly partitions the agent’s interaction with the environment into two distinct stages \citep{jin2020reward,zhang21near,miryoosefi2022simple}. In the first stage, the agent explores the environment to collect a dataset of interactions \citep{azar2017minimax}. In the second, this dataset is used---often via planning---to derive a policy with bounded regret for an arbitrary downstream reward function \citep{jin2020reward}. Although this line of work also studies exploration in the absence of explicit rewards, exploration efficiency is ultimately evaluated in terms of the usefulness of the collected experience for solving externally specified reward functions.

In contrast, our work does not treat exploration as a preparatory phase for later reward optimization. Instead, we avoid reliance on atemporal artifacts such as fixed training and deployment stages---a common practice in the RL literature that has been critically examined by \citet{abel2024three} and \citet{elelimy2025rethinking}. We define exploration efficiency as an intrinsic, ongoing (nonstationary) objective, independent of any explicit utility or reward function, and study its consequences in non-episodic and partially observable settings that do not necessarily fulfill the Markov property.

\paragraph{Open-endedness}
A growing body of work argues that explicit objective optimization can hinder long-term progress by prematurely converging to local optima and failing to discover necessary stepping stones \citep{stanley2015greatness, soros2017open}.
This perspective emphasizes open-ended exploration over direct optimization, suggesting that breakthroughs arise from serendipity and novelty search rather than predefined goals \citep{lehman2011abandoning,kumar2024asal}. The emergence of intelligent behavior via open-ended novelty search has inspired a growing number of works in recent years \citep{bauer2023human,bruce2024genie,matthews2025kinetix}, even characterizing it as essential for superhuman-level intelligence \citep{hughes2024position}.

Most of these approaches require on an automatic process to generate diverse tasks and environments, often framed as Unsupervised Environment Design (UED) \citep{parker2022evolving, bauer2023human, rigter2024rewardfree, beukman2024refining}. Thus, works based on UED and others such as OMNI-EPIC \citep{faldor2025omniepic} or Voyager \citep{wang2024voyager} rely on systems capable of producing large and evolving task distributions. While effective at expanding the space of behaviors an agent may encounter, these methods typically assume substantial control over the environment-generation process. Moreover, directly increasing task diversity also expands the task space, increasing the challenge of identifying tasks that are both informative and learnable.

In contrast, our work does not assume access to environment or task generation mechanisms. Instead, we focus on a fixed environment and study how efficient exploration can sustain the emergence of increasingly complex behavior through the agent's interaction dynamics alone, as shown in Section~\ref{sec:experiments} and discussed in Appendix~\ref{apx:open-endedness}.

\paragraph{Active inference} Since \citet{helmholtz1867handbuch}, a broad class of cognition theories agree that the brain maintains and updates an internal model of the environment \citep{gregory1980perceptions,doya2002bayesian,knill2004bayesian,friston2009brain}. The Free Energy Principle (FEP) \citep{friston2010free} has been proposed as a unifying account of this view. Under the FEP, living agents are characterized as systems that minimize the surprise of their sensory observations---equivalently, a variational free-energy bound---thereby maintaining homeostasis \citep{parr2022active}.
While active inference \citep{parr2022active}, the action-oriented formulation of the FEP, shares some similarities with the present work---notably, behavior that is driven by the reduction of model error---there are fundamental differences. Active inference requires agents to have prior preferences over the causes of sensory inputs, which in the case of living agents can encode preferences over survival-relevant states (e.g., a fish assigning high prior probability to being in water). In contrast, our formulation assumes no such priors. Instead, we argue that behaviors commonly associated with preferences, including those supporting survival, can emerge from efficient exploration alone. Furthermore, efficient exploration evaluates and minimizes model error globally, over the entire set of behaviors that may occur in a given environment (see Definition~\ref{def:gloss}), rather than minimizing expected surprise restricted to trajectories experienced under a particular agent.

\paragraph{Intelligence} There are \textit{many} definitions of intelligence \citep{legg2007collection,agueras2025intelligence}, and this paper does not propose a new one. Nevertheless,
several influential ideas are closely related to the ones presented in this paper and are therefore worth briefly discussing, mainly those by \citet{schmidhuber2008driven} and \citet{chollet2019measure}---note that \citet{silver2021reward} is already discussed in Section~\ref{sec:introduction}.

Intelligence as a process of compression has motivated multiple views on the topic \citet{schmidhuber2008driven,deletang2024language,huang2024compression}.
These are based on the idea that intelligent behavior emerges from an agent’s drive to discover compact representations of its experience.
In particular, \citet{schmidhuber2008driven} proposes compression progress---the intrinsic objective of acting to improve the compressibility of observed data---as a unifying explanation for phenomena such as curiosity, creativity, scientific discovery, and art.
While our framework can be interpreted through a compression lens, it differs fundamentally in what is being compressed. Under compression progress, agents seek actions that immediately improve the compressibility of their lived experience. In contrast, efficient exploration prioritizes experience that improves the compressibility of the environment itself, potentially at the expense of short-term compression gains. This distinction leads to qualitatively different behaviors and objectives, some have already been discussed in
the first paragraph of this section, as \citep{schmidhuber1991possibility} implements approximated compression progress.

Another influential perspective is offered by \citet{chollet2019measure}, who defines intelligence as \textit{“a measure of skill-acquisition efficiency over a scope of tasks.”} Although we do not establish a formal connection between this definition and our efficient exploration formalism, both share important conceptual characteristics. In sufficiently rich environments, efficient exploration requires rapid acquisition of reusable skills. For example, as illustrated by the robot scenario in the introduction, a robot that quickly learns to navigate the hostile island will explore more efficiently and at the same time---under Chollet's definition---be more intelligent than one that does not. At the same time, important differences remain: Chollet's definition is explicitly task-centric, whereas efficient exploration is agnostic to externally specified tasks, i.e., extrinsic objectives.

Finally, we do not claim that efficient exploration constitutes a definition or measure of intelligence. Rather, we view intelligent behavior as a prerequisite for achieving efficient exploration in sufficiently challenging environments, positioning efficient exploration as a driver and beneficiary of intelligence rather than its formal characterization.

\section{Proofs} \label{apx:proofs}

\subsection{Proof of Proposition~\ref{prop:rhounif}} \label{proof:rhounif}

Let $\polunif\in\policysetall$ be the agent that always takes actions uniformly at random. Then, according to Equation~\eqref{eq:prob-h-lambda-e}, the probability over histories of length $k$ that it induces in environment $e$ is given by,
\begin{equation}
  p(\hpre{k} \mid \polunif, \env) = \prod_{i=1}^k \env(o_i\mid \hpre{i-1}, a_i) \polunif(a_i\mid \hpre{i-1}).
\end{equation}
As, $\forall_{h\in\hist, a\in\actions}\, \polunif(a\mid h) = 1 / |\actions| = c$, then, we can write the equation above as,
\begin{equation}
  p(\hpre{k} \mid \polunif, \env) = \prod_{i=1}^k \env(o_i\mid \hpre{i-1}, a_i) \, c  = c^k\, \prod_{i=1}^k \env(o_i\mid \hpre{i-1}, a_i),
\end{equation}
and thus, proving Proposition~\ref{prop:rhounif}.
\subsection{Proposition \ref{prop:equiv-perf-hist}}

\begin{proposition} \label{prop:equiv-perf-hist}
  Let $\policyset^*_t$ be the set of optimally exploring agents for some time horizon $t$ as defined in Equation~\eqref{eq:optimal-explorers}. For all optimal explorers in $\lambda^* \in \policyset^*_t$ , for all history prefixes  $h_{1:i-1}\in\hist_{i-1}$ of length $1 \leq i \leq t$ and actions $a_i,a'_i\in\actions$ such that $\policy^*(a_i\mid h_{1:i-1}) > 0 \wedge \policy^*(a'_i\mid h_{1:i-1}) > 0$, then,
\begin{align}
  \label{eq:constcond}
  & \expect{H_{1:t}}{\sum_{k=i}^t \lossg(\paramoptimi(H_{1:k})) \mid H_{1:i-1} = h_{1:i-1}, A_i = a_i\,; \policy^*, \env} = \nonumber \\
  & \quad \quad \expect{H_{1:t}}{\sum_{k=i}^t \lossg(\paramoptimi(H_{1:k})) \mid H_{1:i-1} = h_{1:i-1}, A_i = a'_i\,; \policy^*, \env}.
\end{align}
\end{proposition}
\begin{proof}
  We prove this by contradiction.
  Suppose that there exists actions $a_i$ and $a'_i$ in the support of $\policy^*(h_{1:i-1})$ such that taking the next action $A_i = a_i$ leads to a higher (less efficient) ECE compared to $A_i = a'_i$, i.e., replace the equality in Equation~\eqref{eq:constcond} with a $>$ inequality. We can construct a new agent $\policy \in \policysetall$ that is identical to $\policy^*$ except that it redistributes the probability mass from $a_i$ to $a'_i$, i.e., $\policy(a'_i \mid h_{1:i-1}) = \policy^*(a_i \mid h_{1:i-1}) + \policy^*(a'_i \mid h_{1:i-1})$ and $\policy(a_i \mid h_{1:i-1}) = 0$.

  By Assumption~\ref{asm:pol-set-coverage}, $\exists \policy \in \policyset$ that behaves identically to $\policy^*$ (in almost all histories up to length $t$) but completely avoids the suboptimal choice $a_i$ in favor of $a'_i$ given $\hpre{i-1}$, strictly improving the ECE, $\ece_t(\policy) < \ece_t(\policy^*)$.  This directly contradicts the premise that $\policy^*\in\policyset^*_t$. Therefore, all actions with positive probability under an optimal explorer yield the exact same expected future cumulative global loss.
\end{proof}


\subsection{Proof of Theorem~\ref{theo:opt-expl-props}} \label{proof:opt-expl-props}

Based on Proposition~\ref{prop:equiv-perf-hist} we prove the properties in Theorem~\ref{theo:opt-expl-props} independently.

\vspace{1mm}
\textit{1)  There always exists at least one fully deterministic optimal explorer: $\forall_{\env\in\envset}\,, \policyset^*_t \cap \detpolicyset \neq \emptyset$.}

As $\policyset$ is a non-empty subset of $\policysetall$ (see Definition~\ref{def:policy}), then, $\policyset^*_t$ is a non-empty subset of $\policyset$. Let $\policy^*\in\policyset^*_t$ be an optimal explorer. If $\policy^*\in\detpolicyset$, the proof trivially concludes.

Let $\policy^*\notin \detpolicyset$. We can always create a deterministic agent $\poldet\in\detpolicyset$ by deterministically selecting an action in the support of $\policy^*(h)$. Namely,
\begin{equation}
  \label{eq:detoptpol}
  \poldet(a\mid h) =
  \begin{cases}
    1 & \text{if $a = \hat{a}$ for some selected $ \hat{a} \in \{a'\in\actions : \policy^*(a'\mid h) > 0\}$}, \\
    0 & \text{otherwise}.
  \end{cases}
\end{equation}

By Proposition~\ref{prop:equiv-perf-hist}, for any history prefix $\hpre{i-1}$, all actions in the support of $\policy^*(\hpre{i-1})$ lead to the exact same future ECE. By backward induction from timestep $t$ to 1, replacing the stochastic action selection over $\policy^*(\hpre{i-1})$ with a deterministic choice of $\hat{a}$ (see Equation~\eqref{eq:detoptpol} above) at each step does not alter the future ECE. Consequently, $\ece_t(\poldet) = \ece_t(\policy^*)$, and thus, $\poldet \in \policyset^*_t \cap \detpolicyset$. This proves the first property in Theorem~\ref{theo:opt-expl-props}.

\vspace{1mm}
\textit{2)  There exist environments with stochastic optimal explorers: $\exists_{\env\in\envset}$ $\policyset^*_t \setminus \detpolicyset \neq \emptyset$.}

We prove this by constructing an environment $\env\in\envset$ where a non-deterministic optimal explorer exists.

Let $\env \in \envset$ where $\poldet^1, \poldet^2 \in \policyset^*_t \cap \detpolicyset$ and $\exists_{\hpre{t}\in\hist_t}$ such that $p(\hpre{t}\mid \polunif, \env) > 0$, and for a prefix $\hpre{i}$ of length $i$ of the history $\hpre{t}$ both agents take the same sequence of actions and then differ,
i.e., $\exists_{i\in [0, t-1]}\, p(\hpre{i}\mid \poldet^1, \env) = p(\hpre{i} \mid \poldet^2, \env) \, \wedge \, p(\hpre{i+1}\mid \poldet^1, \env) \neq p(\hpre{i+1} \mid \poldet^2, \env) $.\footnote{Note that for deterministic agents $p(h\mid \poldet, \env) > 0$ iff $\poldet$ always takes the actions in $h$ given observations in $h$.} Note that at this timestep $i+1$, we can stochastically choose between the action $a^1_{i+1} = \arg \max_{a\in\actions} \poldet^1(a\mid \hpre{i})$ and $a^2_{i+1} = \arg \max_{a\in\actions} \poldet^2(a\mid \hpre{i})$ and still obtain the optimal ECE as long as the rest of the choices are coherent with the chosen agent $\poldet^1$ or $\poldet^2$. In other words, we can construct an agent $\policy \in \policyset^*_t \setminus \detpolicyset$ that given the $k$-th prefix of $\hpre{t}$ act according the following rule,
\begin{equation}\label{eq:branch-det}
  \policy(a\mid \hpre{k}) =
  \begin{cases}
    \poldet^1(a\mid \hpre{k}) & \text{if $k < i$} \\
    p \poldet^1(a\mid \hpre{k}) + (1-p) \poldet^2(a\mid \hpre{k}) & \text{if $k = i$} \\
    \poldet^1(a\mid \hpre{k}) & \text{if $k > i$ and $a_{i+1} = a^1_{i+1}$}  \\
    \poldet^2(a\mid \hpre{k}) & \text{if $k > i$ and $a_{i+1} = a^2_{i+1}$}  \\
  \end{cases}
\end{equation}
where $p \in [0, 1]$ and for any other history that is not $\hpre{t}$ the agent follows either $\poldet^1$ or $\poldet^2$.\footnote{In Equation~\eqref{eq:branch-det}, using $\poldet^1(a\mid \hpre{k})$ or $\poldet^2(a\mid \hpre{k})$ for the first branch leads to the same action choice.} Note that $\policy$ must achieve the same ECE as $\poldet^1$ or $\poldet^2$, as otherwise we reach a direct contradiction with Proposition~\ref{prop:equiv-perf-hist}. Therefore, $\policy\in\policyset^*_t$, completing the proof for the second property in Theorem~\ref{theo:opt-expl-props}.

\paragraph{Examples} Examples of these environments include some symmetric environments, in which the visitation order of some regions of the environment results in equally informative experiences. For instance, visiting an empty room in a clockwise pattern or anti-clockwise pattern can lead to equally informative experiences. Other examples include selecting from multiple \textit{portals} that lead to the same region of the environment, or the trivial case in which $\paramoptimi$ always returns the same function $\wmfunc\in\setwm$, $\forall_{h\in\hist}\, \paramoptimi(h) = f$.



\subsection{Proof of Proposition~\ref{prop:markov-gloss}} \label{apx:markov-gloss}

Given Assumption~\ref{asm:policy1} in Section~\ref{sec:theory}, we have that $\policy(\hpre{t}) = \policy(\hpre{t-1} a_t s_t) = \policy(s_t)$, where we use the $\hpre{t} = \hpre{t-1}a_ts_t$ to denote a history of length $t-1$ where the action $a_t$ and state $s_t$ are concatenated to create a history of length $t$. Similarly, we have that $\wmfunc(s\mid \hpre{t-1}, a_t) = \wmfunc(s \mid s_{t-1}, a_t)$ as $\tranf$ is Markovian (see Definition~\ref{def:cMP}).

Therefore, we can write $\lossg$ in Equation~\eqref{eq:global-loss} as,
\begin{align}
  \lossg (\wmfunc)
  & = \sum_{k=0}^{t-1} \sum_{\hpre{k}} p(\hpre{k}\mid \tranf, \polunif) \sum_a \discrep\big(\wmfunc(\cdot \mid \hpre{k}, a), \, \tranf(\cdot \mid \hpre{k}, a)\big), \nonumber \\
  & = \sum_{k=0}^{t-1} \sum_{\hpre{k}} \sum_a p(\hpre{k} \mid \tranf, \polunif)\, \discrep\big(\wmfunc(\cdot \mid s_{k}, a), \, \tranf(\cdot \mid s_{k}, a)\big).
\end{align}
Note that $s_{k}$ in the equations above is the $k$-th (last) state of the history $\hpre{k}$.
By using $h_{k}^s \in \hist_k^s = \{ h_k\in\hist_k: s_k = s \}$ to denote a history of length $k$ where the last state in the history is $s$, we can write the equation above as,
\begin{align}
  \lossg (\wmfunc) & = \sum_{k=0}^{t-1} \sum_a \sum_s \left(\sum_{\hpre{k}^s\in\hist_k^s}  p(\hpre{k}^s \mid \tranf, \polunif) \right) \discrep\big(\wmfunc(\cdot \mid s, a), \, \tranf(\cdot \mid s, a)\big) \nonumber \\
                   & = \sum_a \sum_s \discrep\big(\wmfunc(\cdot \mid s, a), \, \tranf(\cdot \mid s, a)\big) \sum_{k=0}^{t-1} \sum_{\hpre{k}^s\in\hist_k^s}  p(\hpre{k}^s \mid \tranf, \polunif) \nonumber \\
                   & = \sum_a \sum_s \discrep\big(\wmfunc(\cdot \mid s, a), \, \tranf(\cdot \mid s, a)\big) \underbrace{\sum_{k=0}^{t-1} \Pr(S_k = s \mid \polunif, \tranf)}_{\envweightS(s)}.
\end{align}
Note that, $\envweightS(s)$ is the expected probability of visiting state $s$ in the steps $k\in[1, t)$, and can be further simplified under the Markovian assumption (see Definition~\ref{def:cMP}).
\begin{remark} \label{rem:envweight}
  Let $\mathbf{P} \in [0, 1]^{|\states|\times\|states|}$, where each value is given by,
  \begin{equation}
    \mathbf{P}_{ss'} = \frac{1}{|\actions|} \sum_a \tranf(s'\mid s, a).
  \end{equation}
  Recall, $\hpre{0} = s_0$, thus,
  \begin{equation}
    \envweightS(s) = \sum_{k=0}^{t-1} (\mathds{1}_{s_0} \mathbf{P}^k)_{s}
  \end{equation}
  where $\mathds{1}_{s_0}$ is a vector of zeros of length $|\states|$ except in the index corresponding to $s_0$, where it takes the value 1.
\end{remark}

This completes the proof for Proposition~\ref{prop:markov-gloss}.

\subsection{Proof of Theorem~\ref{theo:main}} \label{apx:proof-theo-lp}
\begin{proof}

Assuming the squared $L_2$ norm as the discrepancy function $\discrep$ in Equation~\eqref{eq:markov-gloss} we can expand the optimization problem from Equation~\eqref{eq:optimal-explorers} as,
\begin{align} \label{eq:obj-argmin}
    \policyset^* & = \underset{\policy\in\policyset}{\arg \min}\ \mathbb{E}_{\rvhist_{1:t} \mid \policy, \tranf} \left[ \sum_{i=1}^t \sum_{s,a} \envweightS(s) \parallel \wmfunc_{\rvhist_{1:i}}(\cdot\mid s, a) - \tranf(\cdot\mid s, a) \parallel_2^2 \right] \nonumber \\
    & = \underset{\policy\in\policyset}{\arg \min}\ \sum_{i=1}^t \sum_{s',a,s}\envweightS(s) \, \mathbb{E}_{\rvhist_{1:t} \mid \policy, \tranf} \left[  \left(\wmfunc_{\rvhist_{1:i}}(s' \mid s, a) - \tranf(s'\mid s, a)\right)^2 \right].
\end{align}

In the equation above, the function $\wmfunc_{H_{1:i}}$ is defined by Equation~\ref{eq:wm-freqs}, that can be written as,
\begin{equation}
  \label{eq:fxn}
  \wmfunc_{\rvhist_{1:i}} = \frac{X + \alpha}{N + \alpha |\states|},
\end{equation}
where $N$ and $X$ are the random variables for $\freq(s, a; \rvhist_{1:i})$ and $\freq'(s', s, a; \rvhist_{1:i})$ respectively. Furthermore, for the sake of clarity, let $p = \tranf(s'\mid s, a)$ for the rest of this section.

\begin{remark}\label{rem:n_binom2}
Conditioned on $N = n$, we have that $X \mid (N = n) \sim \binomial(n, p)$.
\end{remark}

Therefore, by conditioning on $N = n$ the expected value, variance, and bias of the estimator $\wmfunc_{H_{1:i}}$ are given by,
\begin{align}
  \expect{H_{1:t}\mid \lambda, \tranf}{\wmfunc_{H_{1:i}}(s'\mid s, a) \mid N = n} & = \frac{\expect{H_{1:t}\mid \lambda, \tranf}{X \mid N = n} + \alpha}{n + \alpha |\states|} = \frac{n p + \alpha}{n + \alpha |\states|}, \label{eq:fxn-exp} \\
  \variance{H_{1:t}\mid \lambda, \tranf}{\wmfunc_{H_{1:i}}(s'\mid s, a) \mid N = n} & = \frac{\variance{H_{1:t}\mid \lambda, \tranf}{X \mid N = n} }{(n + \alpha |\states|)^2}  = \frac{n p (1-p)}{(n + \alpha |\states|)^2}, \label{eq:fxn-var} \\
  \bias{\wmfunc_{H_{1:i}}(s'\mid s, a) \mid N = n} & = \frac{n p+ \alpha}{n + \alpha |\states|} - p  = \frac{\alpha (1-|\states| p)}{n + \alpha |\states|}. \label{eq:fxn-bias}
\end{align}
Thus, using the bias-variance decomposition on the expected value in objective function from Equation~\eqref{eq:fxn}conditioned on $N = n$,
\begin{align} \label{eq:fxn-bias-var}
  & \mathbb{E}_{\rvhist_{1:t} \mid \policy, \tranf} \left[  \left(\wmfunc_{\rvhist_{1:i}}(s' \mid s, a) - \tranf(s'\mid s, a)\right)^2 \mid N = n \right] \nonumber \\
  & \quad = \variance{H_{1:t}\mid \lambda, \tranf}{\wmfunc_{H_{1:i}}(s'\mid s, a) \mid N = n} + \bias{\wmfunc_{H_{1:i}}(s'\mid s, a) \mid N = n}^2 \nonumber \\
  & \quad = \frac{n p (1-p)}{(n+\alpha |\states|)^2} + \frac{\alpha^2(1-|\states|p)^2}{(n+\alpha |\states|)^2}.
\end{align}
Note that we can recover the original MSE expression from Equation~\eqref{eq:obj-argmin} by taking the expectation over $N$ in Equation~\eqref{eq:fxn-bias-var}. Namely,
\begin{equation}
  \label{eq:expec-n}
  \mathbb{E}_{\rvhist_{1:t} \mid \policy, \tranf} \left[  \left(\wmfunc_{\rvhist_{1:i}}(s' \mid s, a) - p \right)^2  \right] = \expect{N\mid \policy, \tranf}{\mathbb{E}_{\rvhist_{1:t} \mid \policy, \tranf} \left[  \left(\wmfunc_{\rvhist_{1:i}}(s' \mid s, a) - p \right)^2 \mid N = n \right]}.
\end{equation}

\begin{remark}
For the sake of completeness, note that in the Equation~\eqref{eq:expec-n} above,
\begin{equation}\label{eq:prob-N}
     N \sim p(N = n \mid \policy, \tranf) = \expect{H_{1:t} \mid \policy, \tranf}{\mathds{1}(\freq(s, a; H_{1:i}) = n)},
\end{equation}
and the probability distribution $p(\hpre{t} \mid \policy, \tranf)$ over which the expectation is defined is specified in Section~\ref{sec:theory}.
\end{remark}

By replacing the result from Equation~\eqref{eq:fxn-bias-var} into the expression above,
\begin{equation}
  \label{eq:exp-deco1}
  \mathbb{E}_{\rvhist_{1:t} \mid \policy, \tranf} \left[  \left(\wmfunc_{\rvhist_{1:i}}(s' \mid s, a) - p \right)^2  \right] =
  \expect{N\mid \policy, \tranf}{\frac{N p (1-p)}{(N+\alpha |\states|)^2}} + \expect{N\mid \policy, \tranf}{\frac{\alpha^2(1-|\states|p)^2}{(N+\alpha |\states|)^2}}.
\end{equation}
Since $p$, $\alpha$, and $|\states|$ do not depend on $N$, these can be factored out as,
\begin{align}
  \label{eq:exp-deco2}
  & \mathbb{E}_{\rvhist_{1:t} \mid \policy, \tranf} \left[  \left(\wmfunc_{\rvhist_{1:i}}(s' \mid s, a) - p \right)^2  \right] \nonumber \\
  & \quad = p (1-p) \expect{N\mid \policy, \tranf}{\frac{N}{(N+\alpha |\states|)^2}} + \alpha^2(1-|\states|p)^2 \expect{N\mid \policy, \tranf}{\frac{1}{(N+\alpha |\states|)^2}}.
\end{align}
Now by replacing this result in Equation~\eqref{eq:obj-argmin} in the beginning of this section,
\begin{align}
  \policyset^* =\underset{\policy\in\policyset}{\arg \min} \sum_{i=1}^t \sum_{s', a, s} \envweightS(s) \, p (1-p) \expect{N\mid \policy, \tranf}{\frac{N}{(N+\alpha |\states|)^2}} & \nonumber \\
   +\, \alpha^2(1-|\states|p)^2 \expect{N\mid \policy, \tranf}{\frac{1}{(N+\alpha |\states|)^2}} &.
\end{align}
Recall that $p = \tranf(s'\mid s, a)$ and $N = \freq(s, a; \rvhist_{1:i})$. Moreover, the $p (1-p)$ expression above resembles the Gini impurity, $\gini(s,a) = \sum_{s'}\tranf(s'\mid s, a) (1-\tranf(s'\mid s, a))$. Thus, we can equivalently write,
\begin{align}
  \policyset_t^* = \underset{\policy\in\policyset}{\arg \min} \sum_{i=1}^t \sum_{a, s} \envweightS(s) \, \gini(s, a) \, \expect{N\mid \policy, \tranf}{\frac{N}{(N+\alpha |\states|)^2}} \nonumber \\
  +\, \envweightS(s)\, \alpha^2 \unifbias(s,a)\, \expect{N\mid \policy, \tranf}{\frac{1}{(N+\alpha |\states|)^2}},
\end{align}
where $\unifbias(s,a) = \sum_{s'} (1-|\states| \tranf(s'\mid s, a))^2$ is the bias term. This completes the proof of Theorem~\ref{theo:main}.
\end{proof}

\section{Experimentation details} \label{apx:exp-impl}

\subsection{\texttt{SmallWorld} environment} \label{apx:smallworld}

All empirical results reported in this section are obtained in \texttt{SmallWorld}, an environment suite developed specifically for this study. \texttt{SmallWorld} consists of grid-based worlds in which each cell is assigned a scalar value in the range $[-1,1]$. This value determines the cell type: negative values correspond to solid (impenetrable) cells, such as walls or obstacles; zero denotes floor cells; and positive values represent other entities, such as items, collectible objects, or agents. Different agents may be assigned different cell values, which makes them distinguishable from one another.

Observations are real-valued square matrices corresponding to agent-centered slices of the grid. For example, when the observation window has size $7$, each observation is a $7 \times 7$ matrix, and the entry at position $(3,3)$ always contains the value of the agent's own cell. When an agent is close to a boundary of the environment, portions of the observation window that fall outside the grid are filled with a designated padding value of $-1$. Actions are discrete, and all environments considered in this section share the same action set: move up, down, left, or right. Figure~\ref{fig:envs-show} illustrates the environments used in our experiments.

\paragraph{Empty} This environment consists of an empty $10 \times 10$ grid-world containing a single agent with partial observations as described above. The agent's initial position is sampled uniformly at random.

\paragraph{Maze} This environment is a $10 \times 10$ grid containing solid obstacles arranged as a simple maze with four rooms of different sizes. The agent's initial position is randomized.

\paragraph{Blocks} This environment comprises a $16 \times 15$ grid-world that is empty except for three $2 \times 2$ solid blocks. As in the previous environments, the agent's initial position is randomized.

\paragraph{Random Colors} This environment consists of a $15 \times 11$ grid with a large vertical corridor and two rooms located in the upper third of the map. The cell values of both rooms are assigned randomly at each timestep, with all cells within a given room sharing the same value. The left room takes values in $\{0.6, 0.7\}$ with equal probability, whereas the right room takes values in $\{0.6, 0.7, 0.8, 0.9\}$, also uniformly at random. The agent always starts at the bottom-center position shown in Figure~\ref{fig:env-rand-colors}.


Finally, although the visualizations in Figure~\ref{fig:envs-show} are shown in color for illustrative purposes, the underlying environments are not represented using RGB values; instead, they are defined by scalar cell values in the range $[-1,1]$.

\begin{figure}
     \centering
     \begin{subfigure}[b]{0.24\textwidth}
         \centering
         \includegraphics[width=\textwidth]{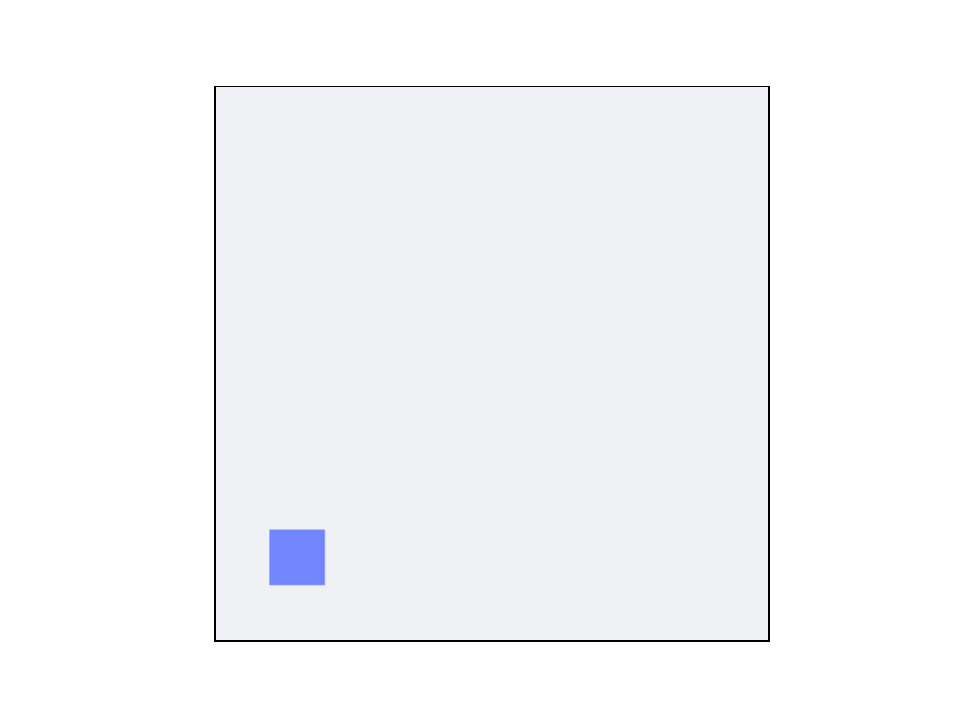}
         \caption{\textit{Empty}}
         \label{fig:env-empty}
     \end{subfigure}
     \begin{subfigure}[b]{0.24\textwidth}
         \centering
         \includegraphics[width=\textwidth]{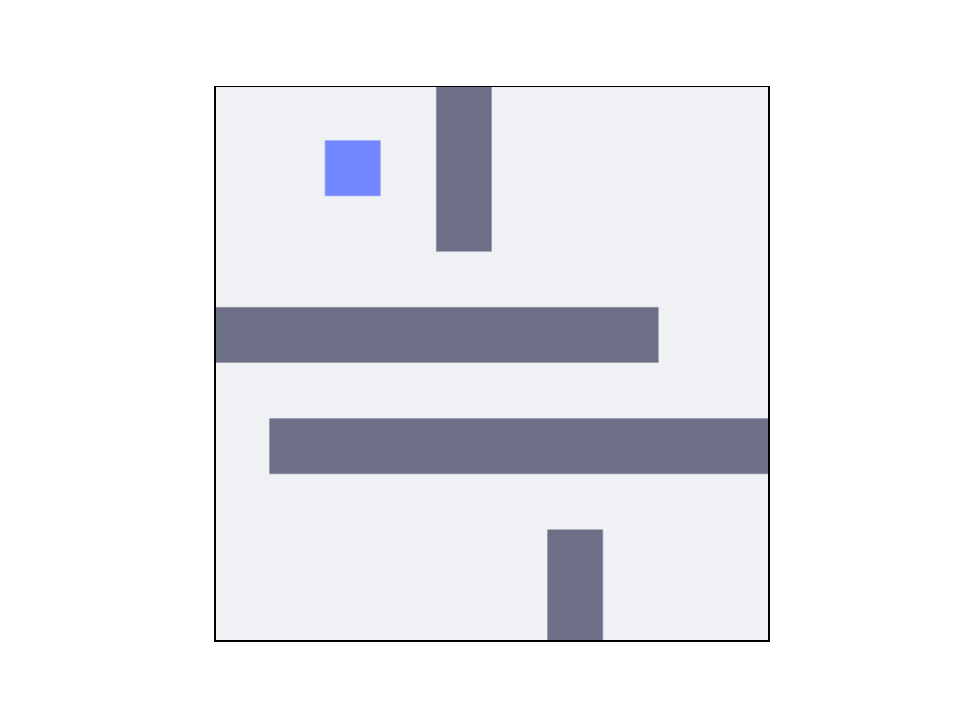}
         \caption{\textit{Maze}}
         \label{fig:env-maze}
     \end{subfigure}
     \begin{subfigure}[b]{0.24\textwidth}
         \centering
         \includegraphics[width=\textwidth]{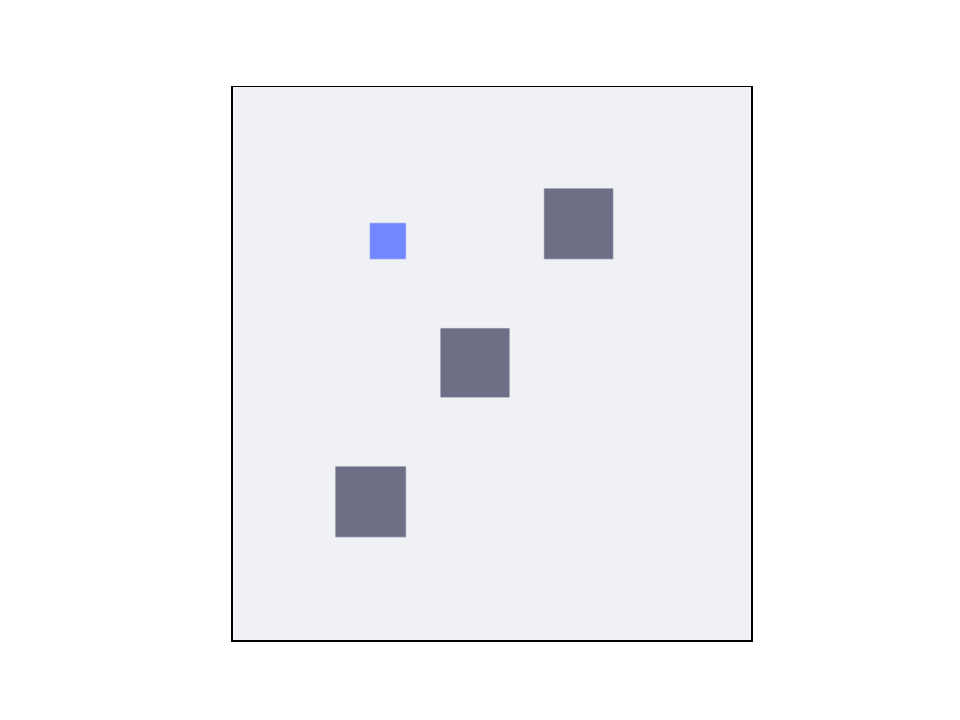}
         \caption{\textit{Blocks}}
         \label{fig:env-blocks}
     \end{subfigure}
     \begin{subfigure}[b]{0.24\textwidth}
         \centering
         \includegraphics[width=\textwidth]{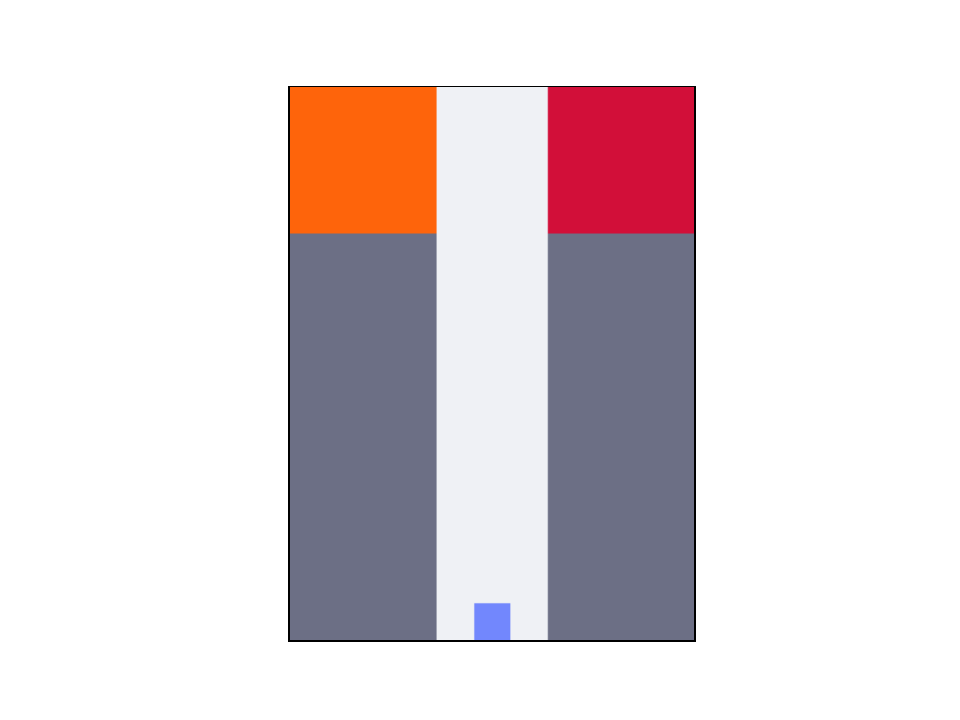}
         \caption{\textit{RandColors}}
         \label{fig:env-rand-colors}
     \end{subfigure}
     \caption{Sample illustrations of the employed environments. The dark gray color \minibox{6c6f85} indicates obstacles, while single cells colored in blue \minibox{7287fd} correspond to agents, the rest of the colors indicate traversable space (i.e., left and right rooms in \textit{RandColors}).}
     \label{fig:envs-show}
   \end{figure}

\subsection{Approximating the ECE} \label{apx:approx-method}

Exactly computing the ECE of an agent (Equation~\eqref{eq:efficiency}) is analytically intractable for most environments of interest, including the relatively simple domains within \texttt{SmallWorld}. As noted in Section~\ref{sec:experiments}, we approximate the ECE via Monte Carlo estimation, averaging over history samples drawn from $p(h \mid \polunif, \env)$ and $p(h \mid \policy, \env)$ to estimate the expectations in Equations~\eqref{eq:global-loss} and \eqref{eq:efficiency}.

However, relying only on trajectories sampled purely from the uniform agent $p(h \mid \polunif, \env)$ introduces some practical challenges. Specifically, in some environments, random walks can predominantly generate relatively predictable trajectories. As the ES progresses and the agents' world models improve, the difference in ECE between competing high-performing agents becomes negligible on most of these trajectories. Consequently, the number of uniform samples required to maintain sufficient evaluation resolution---that is, to accurately distinguish which agents in the population have the lowest ECE---increases.

To mitigate this sample inefficiency, for some environments, we augment the evaluation set used to estimate Equation~\eqref{eq:global-loss} with some histories sampled from systematic deviations of the current population's agents.
Let $\devi : \policyset \rightarrow \policyset$ denote a deviation function that applies a structural modification to an agent \citep{elelimy2025rethinking}.
As the $\polunif$ assigns a non-zero probability to any sequence of actions, any trajectory sampled from $p(h\mid \devi(\policy), \env)$ remains strictly within the support of $p(h \mid \polunif, \env)$. By explicitly drawing some histories from these deviated distributions alongside uniform samples, we ensure that the finite evaluation batch covers regions of the trajectory space where the performance of competing agents actually diverges, stabilizing and improving the convergence of the ES for some environments.

We implement these deviations as simple stochastic perturbations of the agents' behavior: at each step, with probability $\alpha$, the action is sampled uniformly at random, and with probability $1-\alpha$, the agent's prescribed action is executed. Note that a perturbation of $\alpha = 0$ corresponds to the identity function---i.e., executing the unmodified agent---and that if $\alpha = 1$ then $\devi_\alpha(\policy) = \polunif$ (which is used for most cases). We adopt this straightforward perturbation class for simplicity, but richer families of deviations could be integrated for the same purpose \citep{morrill2021efficient}.

\subsection{Method overview} \label{apx:method-overview}

\begin{algorithm}[ht]
\caption{Optimizing towards Efficient Exploration with ES}
\begin{algorithmic}[1]  \label{alg:es_ece}
\REQUIRE Environment $\env$, Population size $N$, Generations $G$, Time horizon $T$, Number of repetitions $R$, Perturbation rates $\mathcal{R} = \{\alpha_1, \dots, \alpha_k\}$
\STATE Initialize Evolutionary Strategy (ES) parameters $\theta$ (e.g., mean and covariance)
\FOR{$g = 1$ \TO $G$}
    \STATE Sample population of agents $\population = \{\policy_1, \dots, \policy_N\}$ using $\text{ES}(\theta)$

    \STATE \algocomment{Construct evaluation set}
    \STATE $\hist_{\text{eval}} \gets \emptyset$
    \FOR{$\policy \in P$}
        \FOR{$\alpha \in \mathcal{R}$}
            \STATE Sample many evaluation histories $\hist_t^{\devi_\alpha(\policy)} \sim p(\cdot \mid \devi_\alpha(\policy), \env)$
            \STATE $\hist_{\text{eval}} \gets \hist_{\text{eval}} \cup \hist_t^{\devi_\alpha(\policy)}$
        \ENDFOR
    \ENDFOR

    \STATE \algocomment{Estimate ECE}
    \FOR{$r=1$ \TO $R$}
        \STATE \algocomment{Collect agent-specific training data}

        \STATE $\mathcal{T} \gets \emptyset$
        \FOR{$\policy \in \population$}
          \STATE Sample history $\hpre{T}^\policy \sim p(h \mid \policy, \env)$
          \STATE $\mathcal{T} \gets \mathcal{T} \cup \{\hpre{T}^\policy\}$
        \ENDFOR

        \STATE \algocomment{Train world models}
        \FOR{$\policy \in \policyset$}
            \STATE $\varepsilon_\policy^r \gets 0$
            \FOR{$k = 1$ \TO $T$}
                \STATE $f_k \gets \paramoptimi(\hpre{k}^\policy)$, where $\hpre{k}^\policy$ is the $k$-th prefix of the training data $\hpre{T}^\policy\in\mathcal{T}$ for agent $\policy$.
                \STATE $\varepsilon_\policy^r \gets \varepsilon_\policy^r +  \frac{1}{|\hist_{\text{eval}}|} \sum_{h \in \hist_{\text{eval}}} \ell(f_k, h)$ \label{line:est-gloss}
            \ENDFOR
        \ENDFOR
    \ENDFOR
    \FOR{$\policy \in \policyset$}
      \STATE $\varepsilon_\policy \gets \frac{1}{R} \sum_{r=1}^R \varepsilon_\policy^r$ \label{line:est-ece}
    \ENDFOR
    \STATE \algocomment{Update population}
    \STATE Update ES parameters $\theta \gets \text{ES\_Update}(\theta, \{\varepsilon_{\policy_1}, \dots, \varepsilon_{\policy_N}\})$
\ENDFOR
\RETURN Agent $\policy^* \in \population$ with the lowest $\varepsilon_{\policy^*}$
\end{algorithmic}
\end{algorithm}

Algorithm~\ref{alg:es_ece} provides an overview of the method used to estimate the ECE and optimize a population of agents for exploration efficiency. Specifically, the global loss in Equation~\eqref{eq:global-loss} is approximated on line~\ref{line:est-gloss}, and the ECE is estimated on line~\ref{line:est-ece}. Note that this pseudocode is a simplification of the highly optimized JAX implementation \citep{jax2018github} used for our experiments. In practice, the history sampling process and world model training are highly parallelized. To further improve computational efficiency, our implementation initializes the training of the current world model using parameters learned during the previous update step. Additionally, rather than training world models for every single history prefix, we train them at batched intervals (of size $t$ divided by the total number of world model training steps, see Table~\ref{tab:hparams}) using progressively more data. Consequently, the world models are updated less frequently than the agent takes steps in the environment.

For the specific components in Algorithm~\ref{alg:es_ece}, the loss function $\ell$ is the Mean Squared Error (MSE) between the observed environment transitions and the world model predictions. The inner-loop optimization of the world models, $\paramoptimi$, uses AdamW \citep{loshchilov2018adamw}. Finally, for the outer-loop evolutionary strategy (ES), we use OpenES \citep{salimans2017evolution} though the implementation provided by the \texttt{evosax} library \citep{lange2023evosax}.

\subsection{Hyperparameters}\label{apx:hparams}

We provide the list of employed hyperparameters in Table~\ref{tab:hparams}. In the case of the view window size of the agents, we maintain a default $5\times 5$ window unless mentioned otherwise. Specifically, we only employ the $7\times 7$ window size in the experiments for the \textit{Empty} environment in Figure~\ref{fig:results-empty-blocks}. Regarding the perturbation probabilities  $\alpha$ (see Appendix~\ref{apx:approx-method}), Table~\ref{tab:alphas} shows the different values employed per environment. Note that for the environments \textit{Blocks} and \textit{Maze} we exclusively employ $\alpha = 1$, which corresponds to sampling histories only by random walks (see Appendix~\ref{apx:approx-method} for details).

As mentioned in Section~\ref{sec:experiments}, we represent agents and world models as NNs. In the case of agents, we employ a single Feed-Forward (FF) layer followed by a GRU module and two FF layers. The input to the agent's network is the flattened $N \times N$ matrix corresponding to the current observation window, whose values are in the $[-1, 1]$ range. For world models, we employ a fully-connected MLP architecture, where the inputs are flattened $S\times N \times N$ observation matrices ($S$ is the sequence length) concatenated with a flattened matrix of $S \times |\actions|$ one-hot encoded actions.

\begin{table}[]
\centering
\caption{Hyperparameters shared across all experiments in Section~\ref{sec:experiments} and Appendix~\ref{apx:more-results}.}
\label{tab:hparams}
\begin{tabular}{@{}cll@{}}
\toprule
                              & \textsc{Name}                   & \textsc{Value}       \\ \midrule
\textsc{General}              & \textsc{Time horizon ($t$)}    & \textsc{512}         \\ \midrule
\multirow{4}{*}{\textsc{Agents}}       & \textsc{Hidden size}            & \textsc{128}         \\
                              & \textsc{Number of layers}      & \textsc{4}           \\
                              & \textsc{Activation function}      & \textit{ReLU}           \\
                              & \textsc{Default view size}     & \textsc{$5\times 5$} \\ \midrule
\multirow{7}{*}{\textsc{World models}} & \textsc{Number of layers}      & \textsc{4}           \\
                              & \textsc{AdamW learning rate}   & \textsc{0.005}       \\
                              & \textsc{Evaluation batch size} & \textsc{512}         \\
                              & \textsc{Training steps}        & \textsc{256}         \\
                              & \textsc{Hidden size}            & \textsc{128}         \\
                              & \textsc{Activation function}      & \textit{ReLU}           \\
                              & \textsc{Sequence length}       & \textsc{8}           \\ \midrule
\multirow{7}{*}{\textsc{OpenES}}       & \textsc{Population size}       & \textsc{128}         \\
                              & \textsc{Number of generations} & \textsc{2000}        \\
                              & \textsc{Number of repetitions} & \textsc{16}          \\
                              & \textsc{Learning rate start}   & \textsc{0.01}        \\
                              & \textsc{Learning rate decay}   & \textsc{0.1}         \\
                              & \textsc{Starting std}          & \textsc{0.05}        \\
                              & \textsc{Std decay}             & \textsc{0.2}         \\ \bottomrule
\end{tabular}
\end{table}

\begin{table}[]
\centering
\caption{List of employed random perturbations for estimating the ECE (see Appendix~\ref{apx:approx-method}).}
\label{tab:alphas}
\begin{tabular}{@{}cc@{}}
\toprule
\textsc{Environment} & $\alpha$ \textsc{values}            \\ \midrule
\textsc{Empty}       & $\{0.1, 0.25, 0.5, 0.75, 1, 1 ,1 ,1\}$ \\
\textsc{Blocks}      & $\{1, 1, 1, 1, 1, 1, 1, 1\}$                       \\
\textsc{Maze}        & $\{1, 1, 1, 1, 1, 1, 1, 1\}$                       \\
\textsc{RandColors}  & $\{0.0, 0.5, 0.75, 1, 1, 1, 1,1\}$       \\ \bottomrule
\end{tabular}
\end{table}

\subsection{Hardware details}\label{apx:hardware}

Experiments have been conducted in a single cluster node with eight Nvidia A5000 GPUs, an AMD EPYC 7252 CPU, and 377GB of RAM. Experiments were parallelized to run across the eight or four GPUs, where a single run required approximately 10 hours.

\subsection{Extended results}\label{apx:more-results}

\begin{figure}
     \centering
     \begin{subfigure}[b]{0.22\textwidth}
         \centering
         \includegraphics[height=25mm]{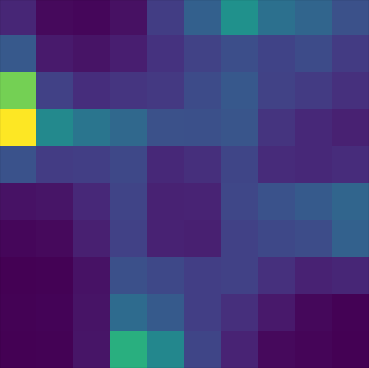}
         \caption{\textit{Empty}}
         \label{fig:res-empty}
     \end{subfigure}
     \hfill
     \begin{subfigure}[b]{0.22\textwidth}
         \centering
         \includegraphics[height=25mm]{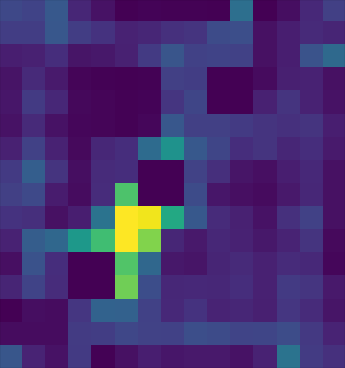}
         \caption{\textit{Blocks}}
         \label{fig:res-blocks}
     \end{subfigure}
     \hfill
     \begin{subfigure}[b]{0.22\textwidth}
         \centering
         \includegraphics[height=25mm]{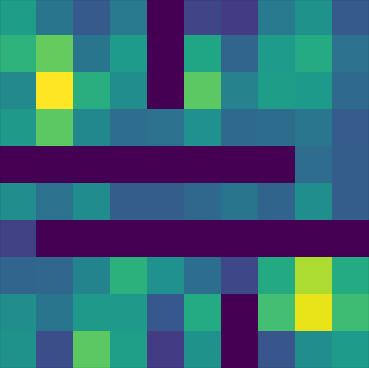}
         \caption{\textit{Maze}}
         \label{fig:res-maze}
     \end{subfigure}
     \hfill
     \begin{subfigure}[b]{0.22\textwidth}
         \centering
         \includegraphics[height=25mm]{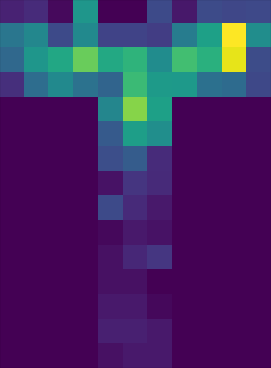}
         \caption{\textit{RandColors}}
         \label{fig:res-randcolors}
     \end{subfigure}
     \caption{Per-cell visitation frequencies of the best exploring agents found by the ES for each \texttt{SmallWorld} environment. Lighter colors indicate more visits.}
     \label{fig:res-sa}
   \end{figure}

In this appendix section we provide additional experiments, figures, and extend discussions on the results shown in Section~\ref{sec:experiments} of the main paper.

Figure~\ref{fig:res-sa} presents the per-cell visitation heatmaps of the most efficient agents discovered by the ES across all four environments (introduced in Section~\ref{sec:experiments} and detailed in Appendix~\ref{apx:smallworld}).
In the \textit{Empty} environment (Figure~\ref{fig:res-empty}), the most efficient exploration strategy converges on a helix-like cyclic pattern that largely avoids both the central region and the immediate edges of the room. The evolution from simpler to more structured behaviors is illustrated in Figure~\ref{fig:fs-curve}.
\begin{wrapfigure}[15]{r}{0.4\textwidth}
    \centering
    \includegraphics[width=\linewidth]{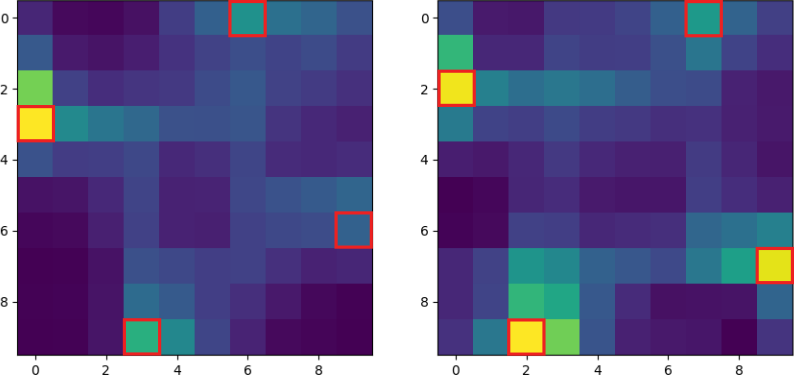}
    \vskip -0.2em
    \caption{Visitation heatmaps of efficiently exploring agents in the \textit{Empty} environment. Both runs share all hyperparameters except the agent's observation window size, which is $7 \times 7$ in the left plot and $5 \times 5$ in the right plot. Red boxes indicate, for each border, the cell with the highest visitation frequency.}
    \vskip -0.2em
    \label{fig:empty-compare}
\end{wrapfigure}
Furthermore, Figure~\ref{fig:empty-compare} illustrates that the geometry of this helix is dependent on the agent's partial observability, i.e., the observation window.
Specifically, the distance between the most frequently visited cells and the nearest wall is consistently $(n - 1)/2$ cells, where $n$ is the size of the observation window.
We hypothesize that this behavior emerges as the center and the extreme edges lead to highly repetitive, easily predictable observations. Instead, the agent actively targets regions that maximize informativeness. The straight traversals forming the helix correspond exactly to paths where the walls fall just outside the agent's field of view. Navigating these specific paths makes predicting the approaching wall significantly more difficult, providing the data necessary for the world model to learn the environment's wall-to-wall dimensions.

\begin{wrapfigure}[16]{r}{0.4\textwidth}
    \vspace{-\baselineskip} 
    \centering
    \includegraphics[width=\linewidth]{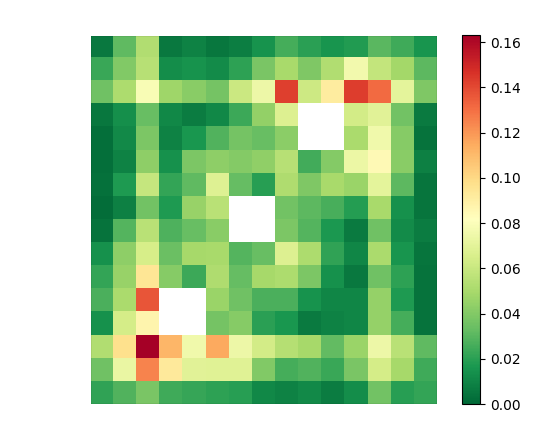}
    \vskip -0.2em
    \caption{World model's error across the \textit{Blocks} environment after training on histories sampled from the behavior shown in Figure~\ref{fig:res-blocks}.}
    \vskip -0.2em
    \label{fig:errors-blocks}
\end{wrapfigure}
The result for the \textit{Blocks} environment is shown in Figure~\ref{fig:res-blocks}. Consistent with the \textit{Empty} results, the agent spends few timesteps in the open spaces between the blocks and the outer boundaries. However, in this case, it concentrates most of its interactions within the region between the lower and middle blocks. Sampling this specific area allows the world model to accurately predict the dynamics of the identical upper gap without requiring direct exploration. The error map in Figure~\ref{fig:errors-blocks} illustrates this generalization, showing that the model achieves low error across the gaps between the blocks while most of the prediction error accumulates in the gaps between the lower and upper blocks and the nearest corners, where the symmetry no longer holds.

In the \textit{Maze} environment, the absence of spatial symmetries prevents the agent from exploiting structural regularities. Consequently, efficient exploration requires comprehensive coverage of the environment, resulting in the approximately uniform visitation distribution shown in Figure~\ref{fig:res-maze}. However, this global coverage is far from a random walk, as can be seen in Figure~\ref{fig:results-maze-randcolors} and discussed in Section~\ref{sec:results1} of the main paper. Instead, agents consistently exploit a highly structured trajectory, systematically navigating the entire maze regardless of its starting position. Upon reaching one end of the maze, the agent reverses course and traverses it in the opposite direction, establishing a continuous, cyclic behavior.

Finally, in the \textit{Random Colors} environment, agents predominantly concentrate interaction on the upper region (both rooms), which presents a greater modeling challenge due to its inherent stochasticity, see Figure~\ref{fig:res-randcolors}. Moreover, as can be observed in the figure, although agents collect data from both rooms, they spend more total time in the rightmost room. This behavior directly reflects the irreducible error of the different regions of the environment: the best achievable error in the right room ($0.0125$) is five times the left room's ($0.0025$). This visiting distribution shows that the agents' overall time allocation is intrinsically shaped by local modeling difficulty, aligning with our theoretical results in Section~\ref{sec:theory}. In conclusion, agents prioritize visiting the most informative and easy to learn regions first (see agent's trajectory in Figure~\ref{fig:results-maze-randcolors}), and schedule the remaining timesteps modeling more challenging regions of the environment . Refer to Section~\ref{sec:results1} for additional discussion on the topic.

\begin{figure}
    \centering
    \includegraphics[width=\linewidth]{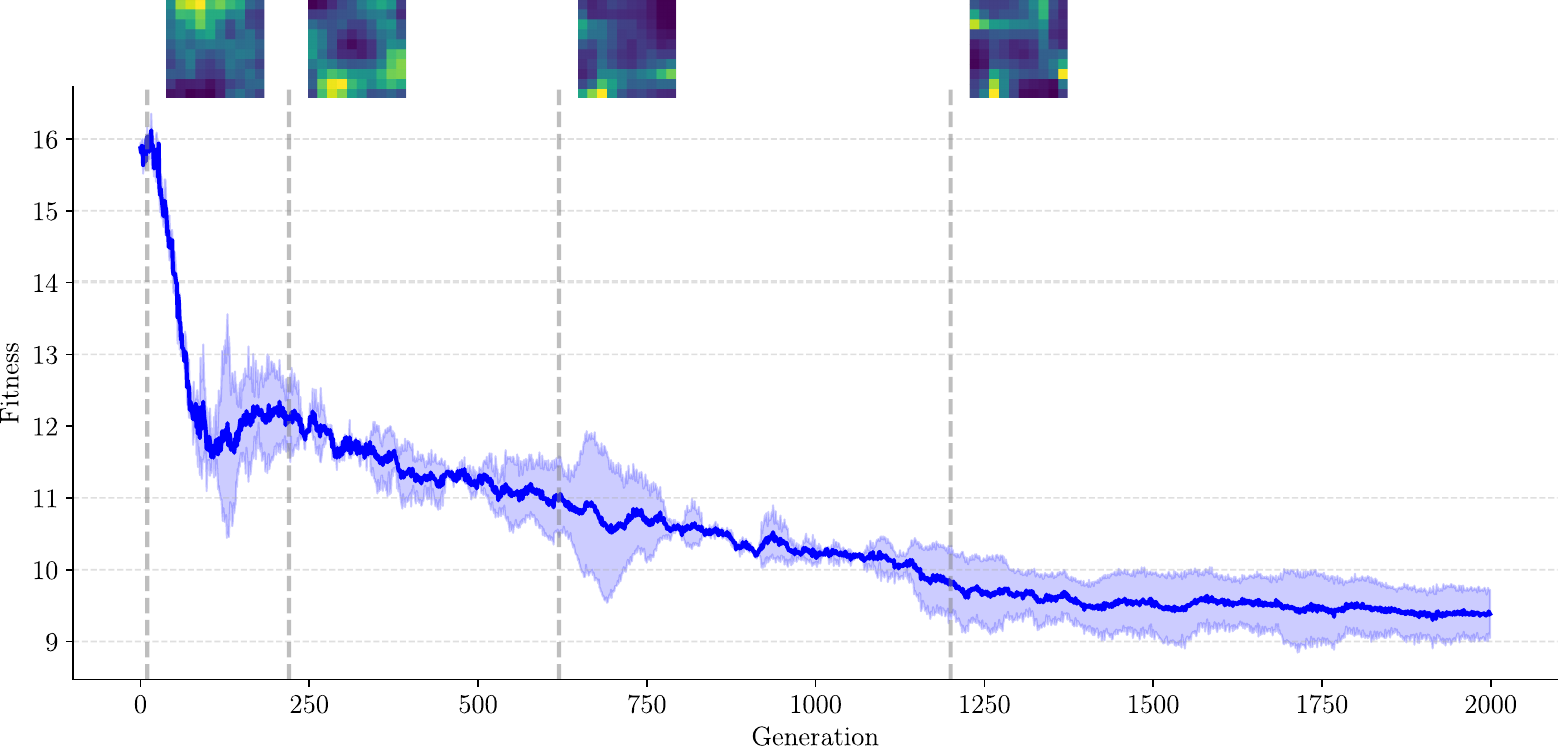}
  \caption{Evolution of the population's average fitness through ES generations in the \textit{Empty} environment with a $5\times 5$ observation window. The blue line indicates the mean value across 3 different seeds, where the shadow corresponds the standard deviation. The figure also shows four different per-cell visitation heatmaps of the most efficient explorers found at the corresponding generations, showing how agents converge to increasingly structured behaviors.}
  \label{fig:fs-curve}
\end{figure}

\begin{figure}
    \centering
    \includegraphics[width=\linewidth]{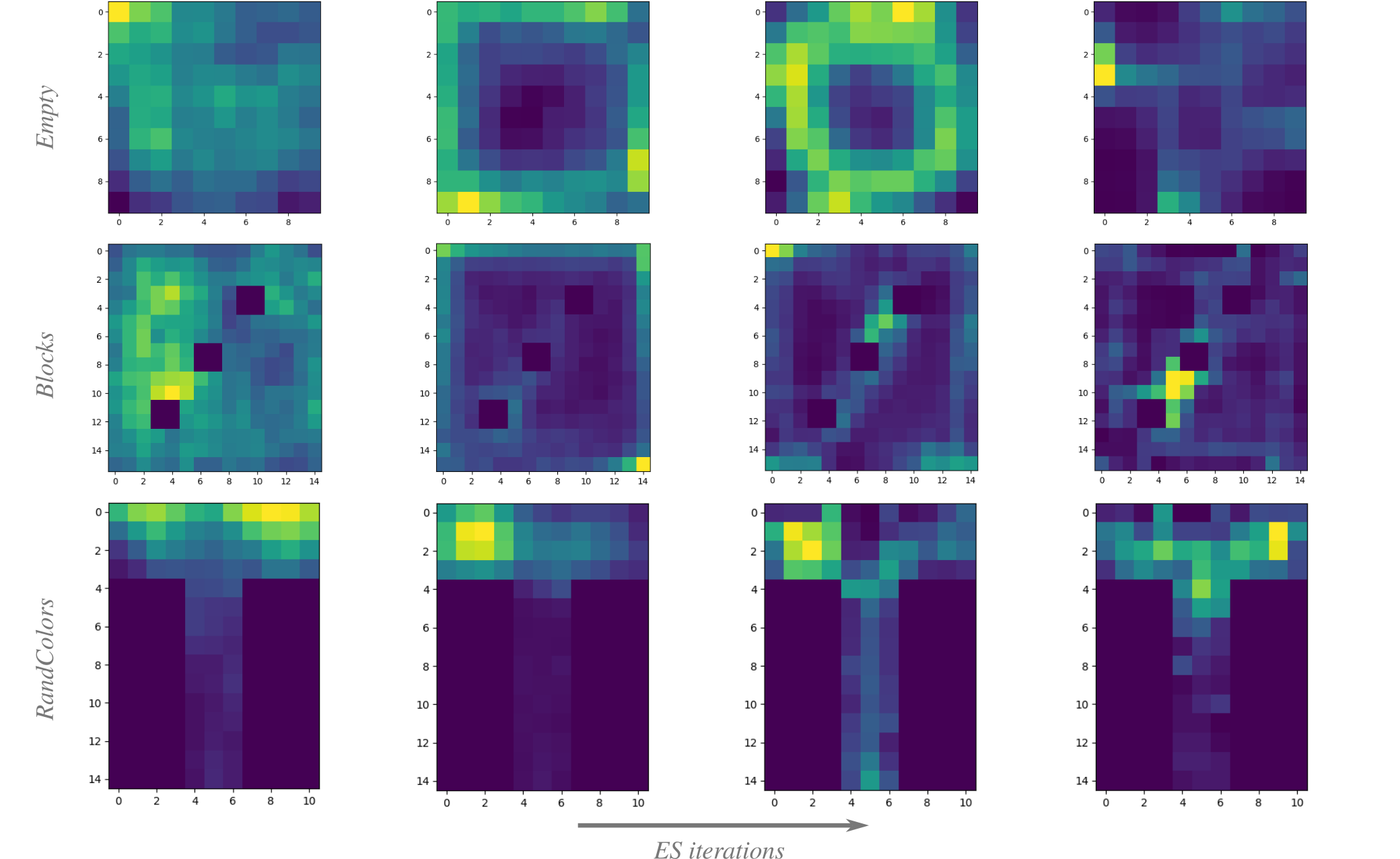}
  \caption{Each row corresponds to one environment, from top to bottom: \textit{Empty}, \textit{Blocks}, and \textit{RandColors}. Columns represent the optimization time of the ES (i.e., generations of OpenES), with the leftmost images corresponding to early generations and the rightmost images to later ones. Each image shows the visitation frequency heatmap of the most efficient explorer found up to the corresponding optimization step (lighter colors indicate more visits).}
    \label{fig:curricula}
    \vskip -1em
\end{figure}

\subsection{Increasingly complex behaviors}\label{apx:increasing-complex}

As discussed in Section~\ref{apx:method-overview}, we use an ES to optimize a population of agents towards (approximate) efficient exploration. By visualizing the behavior of the best agent found at each generation of the ES optimization process, we observe an evolution toward increasingly structured and complex behaviors. Figure~\ref{fig:curricula} illustrates this phenomenon for the \textit{Empty}, \textit{Blocks}, and \textit{RandColors} environments, showing the per-cell visitation frequencies of the best explorers found over the course of optimization (generations of OpenES), from left to right. We omit the \textit{Maze} environment as in this environment, agents converge to uniform coverage (discussed in Appendix~\ref{apx:more-results}).

Starting with \textit{Empty}, we observe that the earliest best strategy consists of approximately uniform coverage of the environment, resembling standard coverage-based exploration methods. As ES optimization progresses, a cyclic, ring-like pattern emerges. This is consistent with the fact that the middle region of the environment is composed mostly of empty cells and is therefore easier to model, leading the agent to allocate more time to harder to model regions such as the corners and edges. The last two plots show how this ring-like pattern progressively rotates until it approaches the helix-like strategy described in Section~\ref{sec:results1}, which achieves higher exploration efficiency than the seemingly more natural circular pattern. This final pattern is discussed more in detail in Appendix~\ref{apx:more-results}. Figure~\ref{fig:fs-curve} illustrates this progression to more structured behavior as the overall ECE of the ES population improves. The figure shows how progressively structured and complex behaviors arise as the average ECE of the ES population decreases (lower corresponds to better exploration efficiency).

A similar progression arises in the \textit{Blocks} environment. As in \textit{Empty}, the first strategy discovered by ES is close to uniform coverage. At later stages, however, better strategies invest more time to the more challenging to model edges and corners of the environment, including the boundaries of the blocks, while largely ignoring the empty regions. In the rightmost plot, the best strategy found by ES strongly prioritizes the region between the lower and middle blocks, thereby exploiting the symmetries of the environment. Notably, the preceding behavior already exhibits an early version of this strategy, but focused on the region between the middle and upper blocks.

Finally, in \textit{RandColors}, the earliest strategy differs qualitatively from the later ones: the agent focuses on the upper part of the environment, while largely ignoring the corridor, which is easier to model. In the second and third plots, the best strategies begin to focus more clearly on the left room. Although the right room is harder to model, concentrating on the left room initially yields a larger reduction in model error, as it is easier for the world model to learn this region. However, as ES continues to optimize agents for efficient exploration, the best strategy eventually allocates more time to the right room than to the left one. Further analysis of the agent's behavior in Section~\ref{sec:results1} shows that the final strategy is also less chaotic: it leaves time to collect data from the right room, but only after efficiently exploring the corridor and then the left room (see Figure~\ref{fig:results-maze-randcolors} and Section~\ref{sec:results1} for details). Thus, in this environment, an efficient explorer spends more time on the hardest to model regions, but only after visiting the easier ones first, yielding the following visitation order: corridor, left room, and right room.

In conclusion, across different environments and settings, optimizing for efficient exploration induces a sequence of related behaviors of progressively increasing complexity. This leads to a strong empirical evidence connecting optimizing towards exploration efficiency and open-endedness as described by \citet{hughes2024position}. Refer to Appendix~\ref{apx:open-endedness} for further discussion.

\section{Extended Discussion}

This section provides an expanded analysis of the practical and conceptual implications of efficient exploration. First, in Appendix~\ref{apx:order-matters} we detail how the sequencing of experience (discussed in Sections~\ref{sec:theory} and \ref{sec:results1})  naturally frames the efficient explorer as an active curriculum generator. Building on this, Appendix~\ref{apx:open-endedness} connects the emergent behaviors produced by efficient exploration at different levels of abstraction (at the interaction and behavioral levels) to the definition of open-endedness by \citet{hughes2024position}.

\subsection{The importance of order in experience and implicit curricula} \label{apx:order-matters}

As demonstrated in Section~\ref{sec:theory} and illustrated in Section~\ref{sec:results1}, the temporal ordering of experience is crucial for efficient exploration (e.g., the \textit{RandColors} results in Figure~\ref{fig:results-maze-randcolors} following the corridor, left room, and right room ordering pattern). Even when multiple regions are equally informative for reducing model uncertainty, the paths required to reach them can differ substantially. Transitioning between informative regions can require traversing less informative, or even already-modeled areas. Consequently, scheduling visitations becomes an optimization problem in its own right, as those in classical routing problems in combinatorial optimization \citep{goldberg1985tsp}. The \textit{Maze} environment perfectly illustrates this dynamic: from its initial position, the agent ignores the equally informative rightmost room to first explore the upper sections. Visiting the rightmost room initially would later force the agent to traverse the starting room a second time to reach the rest of the maze, resulting in redundant and less informative sampling.

In settings more complex than \texttt{SmallWorld}, the choice of the world-model hypothesis space $\setwm$ and the learning rule $\paramoptimi$ would profoundly influence the optimal ordering of experience. The informativeness of any given interaction is not absolute, but it is conditioned on previously acquired knowledge. For example, presenting a child with a textbook on set theory is unlikely to be informative without sufficient prior mathematical background, whereas the same experience may become highly informative at a later stage. This phenomenon is well documented in natural choices for $\setwm$ and $\paramoptimi$---such as NNs optimized via stochastic gradient descent---and has motivated extensive research in curriculum learning \citep{bengio2009curriculum}. From this perspective, an efficient explorer acts as a curriculum generator, constantly sequencing the best possible training data for the specific world model class at hand.

This dependence on curricula also limits how abruptly the global loss $\lossg$ can decrease. Highly informative experiences often require sufficient prior context to be interpretable and therefore useful for learning. This point is well illustrated in \citet{adams1979guide}: although a civilization is advanced enough to compute the answer to the meaning of life, the universe, and everything, it lacks the contextual knowledge required to understand the answer, i.e., $42$. Analogously, in learning systems, the informativeness of an experience is inherently relative to prior experience, which constrains the rate at which global model error---and consequently, the ECE---can be reduced.

Finally, capacity and plasticity constraints of models in $\setwm$ \citep{lyle2023understanding,dohare2024loss,abel2025plasticity} can further increase the importance of order in experience to achieve efficient exploration. With models with limited capacity, agents may need to strategically revisit, reinforce, or even selectively ignore experiences. In such settings, efficient exploration requires not only selecting which experiences to generate, but also deciding to forget them.

\subsection{Open-endedness} \label{apx:open-endedness}

\citet{hughes2024position} formalize open-endedness as a property of a system that continuously generates a sequence of artifacts, evaluated relative to an external observer. Specifically, an artifact-generating system is considered \textit{open-ended} if its outputs are both \textit{novel} and \textit{learnable} to that observer. Novelty implies that, for an observer frozen at time $t$, the artifacts become less predictable over time. Conversely, learnability dictates that conditioning on the historical sequence of artifacts makes future artifacts more predictable. We refer to \citet{hughes2024position} for further details.

Viewed through this lens, we identify two distinct processes within our framework that exhibit open-endedness driven by efficient exploration. The first occurs at the agent-environment interface. As proven in Section~\ref{sec:theory} and empirically validated in Section~\ref{sec:experiments} for more complex settings, an efficiently exploring agent systematically schedules its interactions, sampling the most informative and learnable regions first before transitioning to areas with higher inherent entropy (i.e., more challenging to model). Here, the agent acts as the system generating artifacts (the stream of interactions), while the world model serves as the \textit{observer}. Efficient exploration intrinsically forces these interactions to be novel and learnable for the world model---this is the primary mechanism by which the global loss and ECE are minimized. As noted by \citet{hughes2024position}, an open-ended process can also be temporally bounded. In our setting, beyond the choice of the horizon $t$, this bound is dictated by the agent's and environment's complexity (which limits artifact novelty) and the world model's capacity (which limits learnability).

The second open-ended process emerges at the meta-level during optimization (Section~\ref{sec:results1} and Appendix~\ref{apx:increasing-complex}). Optimizing for efficient exploration naturally yields a progression of increasingly complex behaviors (Figure~\ref{fig:results-empty-blocks}). In this context, the ES operates as the artifact-generating system (producing agent behaviors), while we, as humans, act as the observers. The sequence of behaviors generated by the ES is both novel and learnable: a behavior such as the highly structured, helix-like trajectory in the \textit{Empty} environment might appear unpredictable in isolation, but it becomes more interpretable when viewed as a progression from preceding, simpler behaviors like uniform coverage or ring-like shapes (see Figures~\ref{fig:fs-curve} and \ref{fig:curricula}). This phenomenon is also present in other environments, as shown in Appendix~\ref{apx:increasing-complex}. Without context, understanding the mechanisms driving these advanced behaviors is notoriously difficult. Similar to the first process (described in a lower level of abstraction), this meta-level open-ended process is inherently bounded by the optimization capabilities of the ES (also tied to environment, agent, and world model complexity) and our own capacity to interpret increasingly sophisticated behaviors.

Indeed, even in the relatively simple environments of \texttt{SmallWorld}, the highly structured helix pattern in the \textit{Empty} environment already challenges the limits of intuitive human interpretation. As this open-ended process is scaled to richer, more complex domains, we believe that highly sophisticated behaviors might emerge, possibly surpassing our capacity to fully comprehend their underlying mechanisms.








\end{document}